\documentclass[twocolumn]{article}
\usepackage{etoolbox}
\usepackage{times}

\usepackage{fancyhdr}
\usepackage{xcolor} % changed from color to xcolor (2021/11/24)
\usepackage{algorithm}
\usepackage{algorithmic}
\usepackage{natbib}
\setcitestyle{numbers,square}
\usepackage{eso-pic} % used by \AddToShipoutPicture
\usepackage{forloop}
\usepackage{url}

\usepackage{microtype}
\usepackage{graphicx}
\usepackage{subfigure}
\usepackage{booktabs} % for professional tables

\usepackage{hyperref}

\usepackage{amsmath}
\usepackage{amssymb}
\usepackage{mathtools}
\usepackage{amsthm}

\usepackage[capitalize,noabbrev]{cleveref}

\theoremstyle{plain}
\newtheorem{theorem}{Theorem}[section]

\newtheorem{lemma}[theorem]{Lemma}
\newtheorem{corollary}[theorem]{Corollary}
\theoremstyle{definition}
\newtheorem{definition}[theorem]{Definition}

\theoremstyle{remark}

\usepackage{mathtools}
\usepackage{bbm}
\usepackage{comment}
\DeclareMathOperator*{\argmin}{argmin}

\newcommand{\indic}[1]{\mathbbm{1}[#1]}

\newcommand{\R}{\mathbb{R}}

\newcommand{\data}{\mathcal{D}}
\newcommand{\trainRand}{\mathcal{T}_\textit{rand}}
\newcommand{\trainAdv}{\mathcal{T}_\textit{perturb}}
\newcommand{\X}{\mathcal{X}}

\newcommand{\sample}{\mathcal{S}}

\newcommand{\xb}{x}

\newcommand{\st}{\textnormal{s.t.}}

\newcommand{\clf}[0]{h}
\newcommand{\clfScore}[0]{f}
\newcommand{\ambiguity}[1]{\alpha_{\epsilon}({#1})}
\newcommand{\discrepancy}[1]{\delta_{\epsilon}({#1})}
\newcommand{\churn}[0]{C}
\newcommand{\Hset}[0]{\mathcal{H}}
\newcommand{\epsset}[1]{\mathcal{R}_\epsilon({#1})}
\newcommand{\epssetRand}[0]{\hat{\mathcal{R}}^{m}_\epsilon}
\newcommand{\epssetAdv}[0]{\hat{\mathcal{R}}_\epsilon}
\newcommand{\erm}[0]{\hat{R}}

\newcommand{\unstablesetR}[0]{\sample^{\mathcal{R}}_{\textit{unstable}}}
\newcommand{\unstablesetC}[0]{\sample^{\mathcal{C}}_{\textit{unstable}}}

\usepackage{array,booktabs,tabularx,multirow,colortbl,hhline}
\newcommand{\cell}[2]{\setlength{\tabcolsep}{0pt}\begin{tabular}{#1}#2 \end{tabular}}

\newcolumntype{H}{>{\setbox0=\hbox\bgroup}c<{\egroup}@{}}

\usepackage{tikz}

\usepackage{amsmath,eqnarray,mathtools}
\usepackage{amsfonts, dsfont, bm, courier}
\usepackage{graphicx,xcolor,placeins,caption}
\usepackage{enumitem}
\usepackage{subcaption,afterpage}

\newcommand{\textds}[1]{{\footnotesize\texttt{#1}}}

\newcount\Comments  % 0 suppresses notes to selves in text
\newcommand{\kibitz}[2]{\ifnum\Comments=1{\color{#1}{#2}}\fi}

\newcount\CommentsAdd  % 0 suppresses notes to selves in text
\newcommand{\kibitzAdd}[2]{\ifnum\CommentsAdd=1{\color{#1}{#2}}\fi}

\definecolor{english}{rgb}{0.0, 0.5, 0.0}

\newcommand{\addPredProbsPlot}[1]{%
\includegraphics[]{{#1}}%
}

\title{Predictive Churn with the Set of Good Models}
\author{
    Jamelle Watson-Daniels\thanks{Harvard University},
    \ \ Flavio du Pin Calmon\footnotemark[1],
    \ Alexander D'Amour\thanks{Google DeepMind},\\
    \ \ Carol Long\footnotemark[1],
    \ \ David C. Parkes\footnotemark[1],
    \ \ Berk Ustun\thanks{UC San Diego}
}

\date{}

\begin{document}

\maketitle

\begin{abstract}
Issues can arise when research focused on fairness, transparency, or safety is conducted separately from research driven by practical deployment concerns and vice versa. This separation creates a growing need for translational work that bridges the gap between independently studied concepts that may be fundamentally related. This paper explores connections between two seemingly unrelated concepts of predictive inconsistency that share intriguing parallels. The first, known as \emph{predictive multiplicity}, occurs when models that perform similarly (e.g., nearly equivalent training loss) produce conflicting predictions for individual samples. This concept is often emphasized in algorithmic fairness research as a means of promoting transparency in ML model development. The second concept, \emph{predictive churn}, examines the differences in individual predictions before and after model updates, a key challenge in deploying ML models in consumer-facing applications. We present theoretical and empirical results that uncover links between these previously disconnected concepts.
\end{abstract}

\section{Introduction}
With the widespread use of machine learning (ML) in everyday life, the study of algorithmic fairness has gained prominence, focusing on the social implications of models. A central challenge is ensuring that research in algorithmic fairness is grounded in both social considerations and technical advances. Problems arise when research motivated mainly by fairness, transparency, or safety is developed in isolation from research motivated solely by practical deployment concerns. Conversely, core ML research can advance without adequate attention to fairness and safety, leading to significant issues.

This highlights the need for ``translational" work that bridges the gap between independently studied concepts that may be fundamentally related. Such an interdisciplinary approach can uncover new insights and facilitate valuable knowledge transfer. In this paper, we aim to establish connections between two seemingly unrelated concepts of predictive inconsistency, advocating for a more integrated research approach.

The first concept, \emph{predictive multiplicity}, occurs when models that are ``equally good" on average (e.g., nearly equivalent training loss) yield conflicting predictions for individual samples~\cite{Marx2019}. Predictive multiplicity raises critical questions about model transparency: Are there multiple equally good models that would change an individual's prediction? For example, if different near-optimal models give different loan approval decisions for the same individual, what justifies deploying one model over another~\cite{black2022model}? Researchers in algorithmic fairness emphasize analyzing and reporting predictive multiplicity to enhance accountability and transparency in the ML model development process~\cite{black2022model, Marx2019, Watson-Daniels2022, multitarget, rashomoncapacity, long2023arbitrariness}. By providing information about predictive inconsistency, stakeholders might better gauge trust in model predictions~\cite{Joslyn2013}.

The second concept, \emph{predictive churn}, refers to differences in individual predictions between models before and after updates. This issue is particularly relevant in consumer-facing applications, where unexpected changes due to model updates can lead to adverse effects. Model updates are essential for maintaining and improving long-term performance in mass-market applications like recommendation and advertising. However, in sensitive areas like credit scoring and clinical decision support, changes in predictions can impact customer retention and patient safety.

Consistent, reliable, and predictable behavior is a fundamental expectation of ML models used to support human decision-making. A major challenge in practice is ensuring the stability of predictions following model updates. Our focus is on model updates resulting from changes in training data~\cite{LaunchAndIterate}, though other types of updates are also significant~\cite{hooker2020characterising, DBLP:journals/corr/abs-1910-05446, cooper2021deception, NEURIPS2021_fdda6e95}.

Although research on predictive churn has largely developed independently from fairness considerations~\cite{LaunchAndIterate}, unexpected or unreliable predictions after model updates raise safety concerns, especially when models influence human decision-making. For instance, clinicians can use ML models to support various medical decisions, from diagnosis to prognosis to treatment~\citep{Moreno2005SAPSAdmission, Than2014DevelopmentProtocol, Khand2017HeartValue}. Updates to a medical model, though potentially rendering better \emph{average} performance, may fundamentally impact the treatment selected for individual patients. Generally, addressing predictive inconsistency aligns with the broader idea that deviations from expected behavior can compromise safety~\cite{wei2022safety}.

In this paper, we aim to generate insights by bringing together these two concepts of predictive inconsistency. What can research in algorithmic fairness on predictive multiplicity learn from studies on predictive churn? Conversely, what can industry-focused research on predictive churn gain from understanding predictive multiplicity? We take an initial step in exploring both questions, with greater emphasis on the latter, and suggest that future work further integrate methods from predictive churn research into studies of predictive multiplicity.

The main contributions are:
\begin{enumerate}[itemsep=4mm, leftmargin=*, parsep=0mm]
    \item  We provide theoretical results that establish a connection between two previously disconnected concepts. Specifically, we characterize the expected churn among models within the set of ``good" models from different perspectives. Our analysis demonstrates that the potential to reduce churn by substituting the deployed model with an alternative within this set depends critically on the training procedure used to generate these models. Additionally, we derive an upper bound on churn between ``good" models when considering a model update.

    \item We present empirical evidence that further reveals connections between these concepts. For example, our findings show that analyzing predictive multiplicity can help anticipate churn, even when a model has been enhanced with uncertainty awareness. We also implement an ensemble algorithm, demonstrating that reducing predictive multiplicity can lead to a corresponding reduction in churn.

    \item We empirically investigate whether individual predictions that are unstable due to predictive multiplicity are also unstable due to predictive churn. Our results indicate that the predictive multiplicity ``unstable" set often encompasses most examples within the churn ``unstable" set. Practically, analyzing predictive multiplicity during initial training and test can serve as an early indicator of the potential severity of churn.

\end{enumerate}

\section{Related Work}
\paragraph{Model Multiplicity} Model multiplicity in machine learning often arises in the context of model selection, where practitioners must arrive at a single model to deploy~\citep{Chatfield1995, Breiman2001}, from amongst a set of near-optimal models, known as the ``Rashomon" set. Several studies focused on examining the Rashomon set~\cite{Fisher2019, Dong2019, Semenova2019a, zhong2023exploring, donnelly2023the}. Predictive multiplicity is the prevalence of \emph{conflicting} predictions over the Rashomon set and has been studied in binary classification~\citep{Marx2019}, probabilistic classification~\citep{Watson-Daniels2022, rashomoncapacity}, differentially private training~\citep{Kulynych_2023} and constrained resource allocation~\citep{multitarget}. There is a growing body of research on the implications of differences in models within the Rashomon set~\citep{DAmour2020,veitch2021counterfactual,Pawelczyk2020,Coston2021,Black2021Leave-one-outUnfairness,Ali2021, black2022model, long2023arbitrariness} and on predictive arbitrariness and randomness in a more general setting~\citep{cooper2023prediction, datamultiplicity, prakharRandomness}. Distinctively, the present paper applies the Rashomon perspective to uncover insights about predictive churn.

\paragraph{Predictive Churn} 
Predictive churn is a growing area of research. \citet{LaunchAndIterate} define churn and present two methods of churn reduction: modifying data in future training and regularizing the updated model towards the older model using example weights. Churn reduction is of great interest in applied machine learning~\cite{nondifferentiable, data-constraints,  bahri2021locally}. Distillation~\cite{online-distillation} has also been explored as a churn mitigation technique, where researchers aim to transfer knowledge from a baseline model to a new model by regularizing the predictions towards the baseline model ~\citep{online-distillation, zhang2017deep, on-the-fly, collaborative-learning, jiang2022churn}. Our paper is complementary to this discourse, offering a fresh perspective.

\paragraph{Uncertainty Quantification}
Deep learning uncertainty is often examined from a Bayesian perspective~\cite{Bayesian-methods1992, bayesian-learning-1996}. Many approximate methods for inference have been developed, i.e., mean-field variational inference~\citep{weight-uncertainty2015, pmlr-v108-farquhar20a} and MC Dropout~\citep{gal2016dropout}. Deep ensembles~\citep{deep-ensembles2017} often have comparable performance~\cite{uncertainty-dataset-shift2019} but result in scalability issues at inference time. Predictive uncertainty methods that require only a single model have also been introduced~\citep{prior-networks2018, evidential2018, shu-etal-2017-doc, Bendale2016, single-uncertainties2019, calandra2016manifold, 2022enhanced, riquelme2018deep, scalable-bayes2015, just-a-bit2020, vanamersfoort2020uncertainty, SNGP-2020, plex}; one of which we implement.

\paragraph{Backward Compatibility}
Model update regression or the decline in performance after a model updates~\citep{BCR2022} has been a topic of interest in applied ML~\cite{backward-representation2020}. Researchers have again explored various mitigation strategies including knowledge distillation~\citep{positive-congruent2021, xie2021regression} and  probabilistic approaches~\cite{trauble2021backwardcompatible}. This backward compatibility research is closely related to the concept of \emph{forgetting} in machine learning where some component of learning is forgotten~\citep{lifelong-ML, PARISI201954, biesialska-etal-2020-continual, learn++, gepperth:hal-01418129, class-incremental2023}.

\paragraph{Underspecification and Reproducibility}
Reproducibility is an anchor of the scientific process~\citep{Buckheit1995WaveLabAR, reproducible-2007, open-source2007, encourage-2007, doi:10.1126/science.1179653, doi:10.1126/science.1213847, doi:10.1126/science.1250475, vanschoren2014open, rule2018simple}, and has garnered discussion in ML from the lens of robustness~\cite{cooper2023prediction,DAmour2020}. Recently, research has explored how both reproducibility and generalization relate to ``underspecification"~~\cite{DAmour2020} which is related to overparametrization as well~\citep{doi:10.1073/pnas.1903070116, https://doi.org/10.1002/cpa.22008, nakkiran2019deep}. Our examination of near-optimal models resonates with these studies that explore how the ML pipeline can produce deviating outcomes.

\section{Framework}
\label{sec::framework}

In this section, we define the two types of predictive inconsistency: predictive churn and predictive multiplicity. We begin with a  classification task with a dataset of $n$ instances, $\data{} = \{(\xb_i, y_i)\}_{i=1}^{n}$, where $x_i = [1,x_{i1},\ldots,x_{id}] \in \X \subseteq \R^{d+1}$ is the feature vector and  $y_i \in \{0,1\}$ is an outcome of interest. We fit a classifier $\clf : \R^{d+1} \rightarrow \{0,1\}$ from a hypothesis class $\Hset$ parametrized by $\theta \in \Theta \subseteq \R^d$,
and write $L(~\cdot~; \data{})$ for the {\em loss function}, for example
cross entropy, evaluated on  dataset $\data{}$. Throughout, we let $M(\clf{};\sample)\in \R_+$ denote the performance of $\clf{}\in \Hset$ over a sample $\sample$ in regards to a {\em performance metric} $M(\clf{})$, where we assume lower values of $M(\clf{})$ are better. For instance, when working with accuracy, we measure the \emph{Accuracy error}: $M(\clf{}) = 1 - \textrm{Accuracy}(\clf{})$.

\subsection{Predictive Churn}

\emph{Predictive Churn} considers the differences in predictions between models pre- and post-update. Predictive churn is formulated in terms of two models: a current deployed model, and an updated model resulting from training the current model on additional updated data~\citep{LaunchAndIterate}. Predictive churn is defined over a sample as follows:

\begin{definition}[Predictive churn \cite{LaunchAndIterate} ]
The {\em predictive churn} between two models, $\clf_A$ and $\clf_B$, trained successively on modified training data, is the proportion of examples in a sample $i \in S$ whose prediction differs between the two models:
\begin{align}
  \!  \churn{}(\clf_A, \clf_B; S) \! =\!  \frac{1}{|S|} \sum_{i\in S} \indic{\clf_A(\xb_i) \neq \clf_B(\xb_i)}. \label{eq::churn}
\end{align}
\end{definition}

For simplicity, we use
 shorthand notation $\churn{}(\clf_A, \clf_B)$ in place of $\churn{}(\clf_A, \clf_B; S)$.

In addition to considering churn over a sample, we can consider the set of individual churned examples. If the prediction of an individual example is expected to change as a result of the successive training of a model, then we say the example is \emph{churn unstable}.

\begin{definition}[Churn Unstable Set] \label{churn-unstable}
The {\em churn unstable set} is the set of points in $\sample_{\textit{test}}$ that change over 
a model update from $\clf{}_A$ to $\clf{}_B$, i.e.,
\begin{align*}
  \!  \unstablesetC{}(\clf_A, \clf_B, \sample_{\textit{test}})   \! =\!  \{ i \in \sample_{\textit{test}} \,:\, \clf_A(x_i) \neq \clf_B(x_i) \}
\end{align*}
\end{definition}

\subsection{Predictive Multiplicity}

\emph{Predictive Multiplicity} occurs when models that are ``equally good'' on average (e.g., achieve comparable training loss)  assign conflicting predictions to individual samples~\citep{Marx2019}. Note, the predictive inconsistency is considered over a set of models not just two models. We are interested in conflicting predictions with respect to a set of near-optimal models also referred to as the \emph{$\epsilon$-Rashomon} set of good models~\citep{Marx2019, Watson-Daniels2022, rashomoncapacity}. First, we define the \emph{$\epsilon$-Rashomon} set of good models in two regimes: (i) there exists an optimal model to act as a ``baseline" based on a chosen performance metric (ii) there is no optimal model only a set of ``equally good'' models based on a chosen performance metric. This distinction between how the \emph{$\epsilon$-Rashomon} set is defined will prove useful in \S~\ref{sec::theory}.

\paragraph{Multiplicity with respect to a baseline:} 
The $\epsilon$-Rashomon set 
is defined with respect to a 
{\em baseline model} 
that is 
obtained in seeking a solution to the
 empirical risk minimization problem,
i.e., 
\begin{align}\label{eq::ERM}
    \clf{}_0 \in \argmin_{\clf{} \in \Hset} L(\clf{}; \data{}).
\end{align}

Here, $\clf_0$  denotes the \emph{baseline} classifier.

\begin{definition}[$\epsilon$-Rashomon Set w.r.t. $\clf_0$]\label{def::rashomon}
Given a performance metric $M$, a baseline model $\clf_0$, and error tolerance $\epsilon >0$, the {\em $\epsilon$-Rashomon set} is the set of competing classifiers 
$\clf \in \Hset$ with performance,
\begin{align}\label{eq::epsset}
  \!  \epsset{\clf{}_0} \! :=\!  \{\clf{} \in \Hset: M(\clf{}; \data{}) \leq M(\clf{}_0; \data{})  +  \epsilon\}.
\end{align}
\end{definition}

 $M(\clf{};\mathcal{D})\in \R_+$ denotes the performance of $\clf{}\in \Hset$ over a dataset $\mathcal{D}$ in regards to performance metric,
 $M(\clf{})$. $M(\clf{})$ is typically chosen as the loss function, $M = L(\clf{}; \data{})$, but can also be defined in terms of a direct measure of accuracy~\cite{Watson-Daniels2022}.

\paragraph{Multiplicity without a baseline:} 
For settings without a clear baseline, 
\citet{long2023arbitrariness} suggest an approximation of the Rashomon set, adopted here. This alternative definition involves a randomized training procedure denoted $\trainRand{}(\data{})$ to produce a set of equally good models. For shorthand notation, we leave implicit  in the sequel the dependence of $\trainRand{}$ on the dataset $\data{}$.

\begin{definition}[Empirical $\epsilon$-Rashomon set] \label{def::rand-epsset}
Given a performance metric $M$, an error tolerance $\epsilon >0$, and $m$ models sampled
from
$\trainRand{}$, the {\em Empirical $\epsilon$-Rashomon set}
 is the set of  classifiers $\clf \in \Hset$ with performance metric better than $\epsilon$:

\begin{align}
  \!  \epssetRand{}(\trainRand{}) \! :=\!  
  \begin{split} \{\clf{}_1, \clf{}_2, \cdots \clf{}_m \,:\, \clf{}_k \overset{\mathrm{iid}}{\sim} \trainRand{}, \\
  M(\clf{}_k; \data{}) \leq \epsilon, \forall k \in [m] \}.
  \end{split}
\end{align}
\end{definition}

\paragraph{Predictive Multiplicity Metric: Ambiguity}

\emph{Ambiguity} is a metric used throughout the literature to report predictive multiplicity~\citep{Marx2019, Watson-Daniels2022, rashomoncapacity, multitarget, long2023arbitrariness}. For a dataset sample, ambiguity is the proportion of examples assigned conflicting predictions over the {\em $\epsilon$-Rashomon set} of good models~\citep{Marx2019}. Now, we define ambiguity in the setting of multiplicity without a baseline which is used in the empirical experiments.

\begin{definition}[Empirical $\epsilon$-Ambiguity]
Given the empirical $\epsilon$-Rashomon set, $\epssetRand{}(\trainRand{})$, 
and a dataset sample, $\sample{}$, the 
{\em  empirical $\epsilon$-ambiguity}
 of a prediction problem is the proportion of examples $i \in \sample{}$ assigned conflicting predictions by a  classifier in
 the $\epsilon$-Rashomon set:
\begin{align}\label{eq::ambiguityRand}
  \!  \ambiguity{\epssetRand{}} \! :=\! \frac{1}{|S|} \sum_{i\in S} \max_{\clf{}, \clf{}' \in \epssetRand{}  } \indic{\clf{}(x_i) \neq \clf{}'(x_i)}.
\end{align}
\end{definition}

For simplicity, we use the following shorthand notation $\ambiguity{\epssetRand{}}$ in place of $\ambiguity{\epssetRand{}, \sample}$.

If there exists a model within the $\epsilon$-Rashomon set that changes the prediction of an individual instance,  we say that example is $\epsilon$-Rashomon \emph{unstable} according to Def.~\eqref{def::rand-epsset}.

{\bf Remark.} 
Prior work tends to compute ambiguity over the training set~\cite{Marx2019, Watson-Daniels2022, multitarget}. If $\sample_{\textit{test}}$ is the train dataset, then $\epsilon$-Rashomon \emph{unstable} examples are simply those that are ambiguous according to definitions in the previous section. In experiments, we evaluate unseen test points to determine whether they are $\epsilon$-Rashomon \emph{unstable}.

\section{Methodology}
\label{sec::methods}

In order to explore the relationship between predictive churn and predictive multiplicity (ambiguity), we probe the following questions empirically: How does enhanced uncertainty quantification (model type) relate to churn and ambiguity severity? Does one anticipate the other, i.e., what is the intersection between the $\epsilon$-Rashomon \emph{unstable} set and the \emph{churn unstable} set? Can we predict churn directly? Do ambiguity reduction methods also reduce churn? Now, we detail the methods used to examine each of these questions.

\subsection{Enhanced Uncertainty Quantification}
We aim to understand whether a model type with enhanced uncertainty quantification or uncertainty awareness (UA) can help identify such unstable examples. Given that Bayesian approaches can be computationally prohibitive when training neural networks, methods have been proposed for uncertainty estimation that require training only a single deep neural network (DNN)~\citep{prior-networks2018, evidential2018, shu-etal-2017-doc, Bendale2016, single-uncertainties2019, calandra2016manifold, 2022enhanced, riquelme2018deep, scalable-bayes2015, just-a-bit2020, vanamersfoort2020uncertainty, SNGP-2020, plex}; in particular, we implement the Spectral-Normalized Neural Gaussian Process (SNGP) method~\cite{SNGP-2020} given its widespread use in industry settings.

\citet{SNGP-2020} propose \emph{Spectral-normalized Neural Gaussian Process} (SNGP) for leveraging Gaussian processes in support of distance awareness. The Gaussian process is approximated using a Laplace approximation, resulting in a closed-form posterior for computing predictive uncertainty. SNGP improves distance awareness by ensuring that (1) the output layer is distance aware by replacing the dense output layer with a Gaussian process and (2) the hidden layers are distance preserving by applying spectral normalization on weight matrices. In our experiments, we implement both a standard DNN and a DNN updated with the SNGP technique (DNN-UA for uncertainty-awareness). 

\subsection{Intersection between Unstable Sets}
Given a fixed $\sample_{\textit{test}}$, we can compare $\unstablesetR{}$ and $\unstablesetC{}$ to characterize the relationship between predictive multiplicity and predictive churn. To do this, we train the empirical $\epsilon$-Rashomon set to identify the $\epsilon$-Rashomon \emph{unstable} set of examples in $\sample_{\textit{test}}$. We also simulate a dataset update to identify the \emph{churn unstable} set of examples in $\sample_{\textit{test}}$. Finally, we calculate the intersection between $\unstablesetR{}$ and $\unstablesetC{}$ for the fixed $\sample_{\textit{test}}$.

Additionally, we can use the fixed $\sample_{\textit{test}}$ and the identified unstable points for direct prediction. Given a sample and the accompanying unstable set $\unstablesetC{}$, we can train a classifier to predict whether an example will likely be in the unstable set. We construct a simple classification task with a dataset of $n$ instances, $\data{} = \{(\xb_i, y_i)\}_{i=1}^{n}$, where $x_i = [1,x_{i1},\ldots,x_{id}] \in \X \subseteq \R^{d+1}$ is the feature vector and  $y^c_i \in \{0,1\}$ is now the label indicating whether the example churned (i.e. $\indic{x_i \in \unstablesetC{}}$). We can measure the linear relationship or correlation between variables by analyzing the Pearson Correlation for each configuration. We are particularly interested in the correlation between the different feature configurations and churn.

\subsection{Ambiguity Reduction \& Churn}
\citet{long2023arbitrariness} present a simple ensemble algorithm for ambiguity reduction and detail theoretical guarantees to show that ambiguity is reduced. The ensembling process involves training each model via $\trainRand{}$, then combining those individual predictions to produce a combined prediction. The set of models that is averaged over is exactly an empirical $\epsilon$-Rashomon set of models.

\begin{definition}[Ensemble Classifier~\cite{long2023arbitrariness}]
Given the set of models, $\epssetRand{}(\trainRand{})$, and a vector $\lambda \in \Delta_m$, the ensemble classifier is the convex combination $\clf{}^{\lambda} := \sum_{j \in [m]} \lambda_j \clf{}_j$

where $\clf{}_j$ is the $jth$ model from $\epssetRand{}(\trainRand{})$.
\end{definition}

For our analysis, we assume the weights $\lambda \in \Delta_m$ to be the vector $\frac{1}{m}$. See \citet{long2023arbitrariness} for details on parameter optimization.

To calculate ambiguity, we train multiple ensembled classifiers and then determine whether there is predictive disagreement among them. Of course, in the large ensemble limit, the disagreement between ensembles becomes zero. In practice, we use a finite ensemble due to the limited computational cost.

\section{Theoretical Results}
\label{sec::theory}
This section provides theoretical insights into churn using the multiplicity perspective. Our goal is to outline theoretical connections between the two previously disconnected concepts. Accompanied proofs are in the Appendix. Below, we summarize the implications of the results in this section.

We assume that a practitioner can only access the initial Model $A$. In \S~\ref{sec::theory-3}, we derive an analytical bound on the expected churn between Model $A$ and a prospective Model $B$ using only the properties of their respective Rashomon sets. This result implies that the expected churn will be nicely bounded if future models are confined to the $\epsilon$-Rashomon set (with respect to a baseline).

Again, operating under the premise that we only have access to Model $A$, we analyze whether one model within the $\epsilon$-Rashomon set might result in less churn compared to another model within the $\epsilon$-Rashomon set. Specifically, we aim to quantify the expected churn difference between any two models within the $\epsilon$-Rashomon set. In \S~\ref{sec::theory-1}, we assume that the $\epsilon$-Rashomon set is defined with respect to a {\em baseline model} and derive an expected churn difference that resembles prior bounds on discrepancy (Def.~\ref{eq::discrepancy}) a metric from predictive multiplicity~\cite{Marx2019, Watson-Daniels2022}. In \S~\ref{sec::theory-2}, we operate without a baseline and show that the expected churn difference between two models within the $\epsilon$-Rashomon set can be negligible. These results underscore that the feasibility of mitigating churn by substituting Model $A$ with an alternative from the $\epsilon$-Rashomon set depends on the methodology used to construct the $\epsilon$-Rashomon set, particularly the presence of a baseline model.

\subsection{Expected Churn Between Rashomon Sets $\epsset{\clf{}_0}$}\label{sec::theory-3}

Consider an $\epsilon$-Rashomon set 
with respect to a baseline model, $\epsset{\clf{}_0}$. Say we have two training datasets $\data_A$ and $\data_B$ where $\data_B$ is an updated version of $\data_A$, and consider $\epsset{\clf{}_0^A}$ and $\epsset{\clf{}_0^B}$ respectively (where the baseline is defined according to Eq.~\eqref{eq::ERM} and Eq.~\eqref{eq::epsset})
% \dcp{need to define these baselines; e.g., are they optimal for empirical risk on training sets, or something else?}

We ask what the maximum difference in churn will be between two models from each $\epsilon$-Rashomon set; i.e.,  we want to find the worst case scenario in terms of churn between
 $\epsset{\clf{}_0^A}$ and $\epsset{\clf{}_0^B}$.
We begin by restating a bound on churn between two models, 
making use of smoothed churn alongside $\beta$-stability~\citep{NIPS2000_49ad23d1} of algorithms defined here.

\begin{definition}[$\beta$-stability \citep{LaunchAndIterate}]\label{def::beta-stability}
    Let $f_T(x)\mapsto \mathbf{R}$ be a classifier discriminant function (which can be thresholded to form a classifier) trained on a set $T$. Let $T^i$ be the same as $T$ except with the $i$th training sample $(x_i,y_i)$ replaced by another sample. Then, as in \citep{NIPS2000_49ad23d1}, training algorithm $f(.)$ is $\beta$-stable if:
    \vspace{0em}
    \begin{align}
        \forall x, T,T^i: |f_T(x)-f_{T^i}(x)|\leq \beta
    \end{align}
\end{definition}

%% Defining smooth churn
We begin by following \citet{LaunchAndIterate} to define \emph{smooth} churn and additional assumptions. These assumptions allow us to rewrite churn in terms of zero-one loss:
%
% $$ \churn{}(\clf{}_A, \clf{}_B) = \mathbbm{E}_{(X,Y) \sim \data} \left[ \ell_{0,1}(\clf{}_A(X), Y) - \ell_{0,1}(\clf{}_B(X), Y) \right].$$
% \vspace{-1em}
% \begin{align*}
% \churn{}&(\clf{}_A, \clf{}_B) &= 
% &\quad \mathbbm{E}_{(X,Y) \sim \data} \left[ \ell_{0,1}(\clf{}_A(X), Y) - \ell_{0,1}(\clf{}_B(X), Y) \right],
% \end{align*}
\begin{equation*}
    \churn{}(\clf{}_A, \clf{}_B) = \mathbbm{E}_{(X,Y) \sim \data} \left[ \ell_{0,1}(\clf{}_A(X), Y) - \ell_{0,1}(\clf{}_B(X), Y) \right]
\end{equation*}

This requires that the data perturbation (update from $\data_A$ to $\data_B$) does not remove any features, that the training procedure is independent of the ordering of data examples, and  that training datasets are sampled i.i.d., which ignores dependency between successive training runs. 

\citet{LaunchAndIterate} also introduce a relaxation of churn 
called \emph{smooth churn}, which is parametrerized by $\gamma>0$, and defined as
%
% $$ \churn{}_{\gamma}(\clf{}_A, \clf{}_B) = \mathbbm{E}_{(X,Y) \sim \data} \left[ \ell_{\gamma}(\clf{}_A(X), Y) - \ell_{\gamma}(\clf{}_B(X), Y) \right],$$
% %
% where $\ell_{\gamma}$ is defined as
% %
% \vspace{-1em}
% \begin{align*}
%  \ell_{\gamma}(f(X), f'(X)) =
% \left\{
% \begin{tabular}{ll}
%  $1$,  &  \mbox{   if  $f(X)f'(X) < 0$,}
% \\
%  $1 - \frac{f(X)f'(X)}{\gamma}$, &\mbox{if  $0 \leq f(X)f'(X) \leq \gamma$,}
% \\
%   $0$, &\mbox{otherwise.}
%   \end{tabular}
%   \right.
%   \end{align*}

% \begin{align*}
% \churn{}_{\gamma}&(\clf{}_A, \clf{}_B) = \\
% &\quad \mathbbm{E}_{(X,Y) \sim \data} \left[ \ell_{\gamma}(\clfScore{}_A(X), Y) - \ell_{\gamma}(\clfScore{}_B(X), Y) \right],
% \end{align*}
\begin{equation*}
    \churn{}_{\gamma}(\clf{}_A, \clf{}_B) = \mathbbm{E}_{(X,Y) \sim \data} \left[ \ell_{\gamma}(\clfScore{}_A(X), Y) - \ell_{\gamma}(\clfScore{}_B(X), Y) \right]
\end{equation*}

\text{where } $f_\cdot(X) \in [0,1]$ is a score that is thresholded to produce the classification $h_\cdot(X)$, \text{and } $\ell_{\gamma}$ \text{ is defined as}
% \begin{align*}
% \ell_{\gamma}(f(X),& Y) = \\
% &\left\{
% \begin{array}{ll}
% 1, & \text{if } f(X)Y < 0, \\
% 1 - \frac{f(X)Y}{\gamma}, & \text{if } 0 \leq f(X)Y \leq \gamma, \\
% 0, & \text{otherwise.}
% \end{array}
% \right.
% \end{align*}
\begin{equation*}
\ell_{\gamma}(f(X), Y) = \left\{
\begin{array}{ll}
1, & \text{if } f(X)Y < 0, \\
1 - \frac{f(X)Y}{\gamma}, & \text{if } 0 \leq f(X)Y \leq \gamma, \\
0, & \text{otherwise.}
\end{array}
\right.
\end{equation*}

where  $Y \in \{0,1\}$ here.\footnote{Other formulations define $f_\cdot(X) \in [-1,1]$ and $Y \in \{-1,1\}$. The definitions work for this setup as well. In our case, while the piece-wise function is constrained to the positive quadrant, the loss function adapts appropriately.}

Here, $\gamma$ acts like a confidence threshold. We can use 
smoothed churn alongside
the $\beta$-stability~\citep{NIPS2000_49ad23d1} (see Definition \ref{def::beta-stability}) of algorithms following~\citep{LaunchAndIterate} to derive the bound on expected churn between models within an $\epsilon$-Rashomon set.

\begin{theorem}[Expected Churn between Rashomon Sets]\label{thm::churn-rashomon} 
Assume a training algorithm that is $\beta$-stable. Given two $\epsilon$-Rashomon sets defined with respect to the baseline models, $\epsset{\clf{}_0^A}$ and $\epsset{\clf{}_0^B}$, 
the smooth churn between any pair of models within the two $\epsilon$-Rashomon sets: $\clf{}'_A \in \epsset{\clf_0^A}$ and $\clf{}'_B \in \epsset{\clf{}_0^B}$ is bounded as follows:
\begin{align}
    \mathbbm{E}_{\mathcal{D}_A, \mathcal{D}_B \sim \mathcal{D}^m}[C_{\gamma}(\clf{}'_A, \clf{}'_B)]\leq \frac{\beta \sqrt{\pi n}}{\gamma} + 2\epsilon.
\end{align}

\end{theorem}

This holds assuming all models $\clf{}$ are trained with randomized algorithms (see discussion in appendix) which are also $\beta$-stable (Def.~\ref{def::beta-stability}). % At a high level, this means that if one could restrict prospective models to those only amongst $\epsilon$-Rashomon set of models, then the expected churn will be nicely bounded.

\subsection{Churn for Models within $\mathcal{R}_\epsilon$}
\label{sec::theory-1}
We bound the churn between an optimal baseline model 
% \dcp{note: here it is important for corr 5.2 that $h_0$ is optimal, and not just a baseline. can you also state a result for $h_0$ being any model in the eps rashomon set?} 
and a model within the $\epsilon$-Rashomon set. Let $\erm$ denote empirical risk (error) where $\erm := \frac{1}{n} \sum_i \indic{\clf{(\xb_i} \neq y_i)}$.

% \anote{Are we assuming that $h_A$ and $h_B$ are members of the $\epsilon$-Rashomon set? Is there no dependence on how the training data for $h_A$ and $h_B$ differ?} \jwd{while within one Rashomon set, no. But yes once we have two (one for A and one for B) like Section 4.3 then yes.}

\begin{lemma}[Bound on Churn]\label{Prop::churn_bound}
    The churn between two models $\clf_1$ and $\clf_2$ is bounded by the sum of the empirical risks of the models:
    
    \begin{equation}
        \churn{}(\clf_1, \clf_2) \! \leq \erm{}(\clf_1) + \erm{}(\clf_2).
    \end{equation}
\end{lemma}

\begin{corollary}[Bound on Churn within $\mathcal{R}_\epsilon$]\label{cor::churn-bound}
    Given a baseline model, $\clf{}_0$, and an $\epsilon$-Rashomon set, $\epsset{\clf{}_0}$, the churn between $\clf{}_0$ and any  classifier in the $\epsilon$-Rashomon set, $\clf{}' \in \epsset{\clf{}_0}$, is upper bounded by: 
    
        \begin{align}
        \!  \churn{}(\clf{}_0, \clf{}') \! \leq 2 \erm{}(\clf{}_0) + \epsilon.
    \end{align}
\end{corollary}

We have recovered a bound on churn that resembles the bound on discrepancy derived in~\citep{Marx2019} where they show that the discrepancy between the optimal model and a model within the $\epsilon$-Rashomon set will obey $\discrepancy{\clf{}_0} \leq 2 \erm{}(\clf{}_0) + \epsilon$.

\subsection{Expected Churn within $\epssetRand{}(\trainRand{})$}
\label{sec::theory-2}
Consider a randomized training procedure $\trainRand{}(\data{})$ over a hypothesis class $\Hset$ and a fixed finite dataset $\data$. Say we derive the empirical $\epsilon$-Rashomon set, $\epssetRand{}(\trainRand{})$, according to Def.~\ref{def::rand-epsset}. We ask whether there is
 a model within this empirical $\epsilon$-Rashomon set 
 that might decrease churn if used as an alternative starting point for the successive training of two models.
 % \dcp{in what sense ---- if used as the starting point for a successive sequence of trained models, or if substituted in at each successive step, in place of the current one just generated in training?}. 
Said another way, we are interested in whether switching one model out for another within the $\epsilon$-Rashomon set will impact churn.

Given $\trainRand{}(\data{})$ is a randomized training procedure, we show there is no difference in expected churn when adopting any two models in $\epssetRand{}(\trainRand{})$ as $\clf{}_A$ and $\clf{}_A'$, and considering churn with respect to some other model $\clf{}_B$.
%
% When choosing a method to approximate the $\epsilon$-Rashomon set, this is something to keep in mind.

%
\begin{lemma}[Same Expected Churn within $\epssetRand{}(\trainRand{})$]\label{lem::churn-diff}
Assume a randomized training procedure $\trainRand{}(\data{})$. Fix a training dataset $\data_A$ and an arbitrary model $h_B$. Let $\clf{}_A$ and $\clf'_A$ be two models induced by $\trainRand{}(\data{}_A)$. 
The expected difference in churn between any models $\clf{}_A$ and $\clf'_A$ induced by $\trainRand{}(\data{}_A)$ is zero

\begin{align*}
    \mathbbm{E}_{\clf{}_A, \clf{}'_A \stackrel{iid}{\sim} \trainRand{}(\data{}_A)}
    \left[\churn{}(\clf{}_A, \clf{}_B)  - \churn{}(\clf'_A, \clf{}_B) \right] = 0
\end{align*}

\end{lemma}

This means that one model sampled from $T_{rand}$ will have the same expected churn as another model sampled from $T_{rand}$. In essence, we will not reduce churn by replacing the current model with one from the $\epsilon$-Rashomon set when using the randomized approximation approach.

\section{Empirical Results}
\label{sec::results}

This section presents experiments on real-world datasets in domains where predictive instability is particularly high-stakes (i.e., lending, housing, medicine). 

\paragraph{Setup}

\begin{table}[t!]
\centering
{\small 
\resizebox{0.9\linewidth}{!}{%
\begin{tabular}{l p{0.2\linewidth} rrr}
\toprule
Dataset Name & Outcome Variable & $n$ & $d$ & Class Imbalance \\
\midrule
Adult \citep{misc_census_income_20} & person income over \$50,000 & 16,256 & 28 & 0.31 \\ 
HMDA \citep{cooper2023prediction} & loan granted & 244,107 & 18 & 3.3 \\
Credit \citep{Yeh2009} & customer default on loan & 30,000 & 23 & 3.50 \\
Mammo \citep{Elter2007} & mammogram shows breast cancer & 961 & 12 & 0.86 \\
\bottomrule
\end{tabular}%
}}
\caption{Datasets used in the experiments. For each dataset, we report $n$, $d$ and the class imbalance ratio of a model on test data to demonstrate the diversity in dataset characteristics.}\label{Table::Datasets}
\end{table}

\begin{figure*}[!t]
     \centering
    %  \scriptsize
     \resizebox{0.9\linewidth}{!}{\begin{tabular}{ccc}
     \addPredProbsPlot{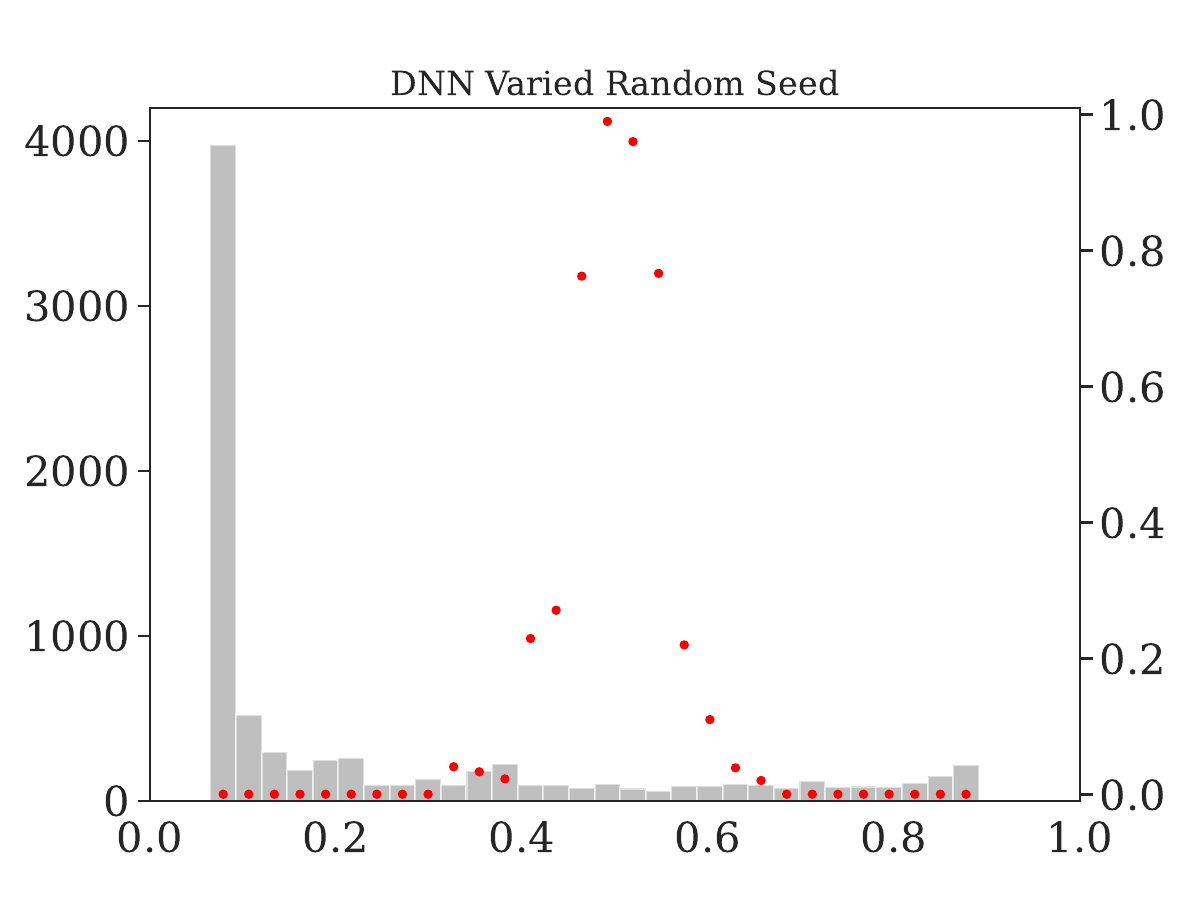} &
     \addPredProbsPlot{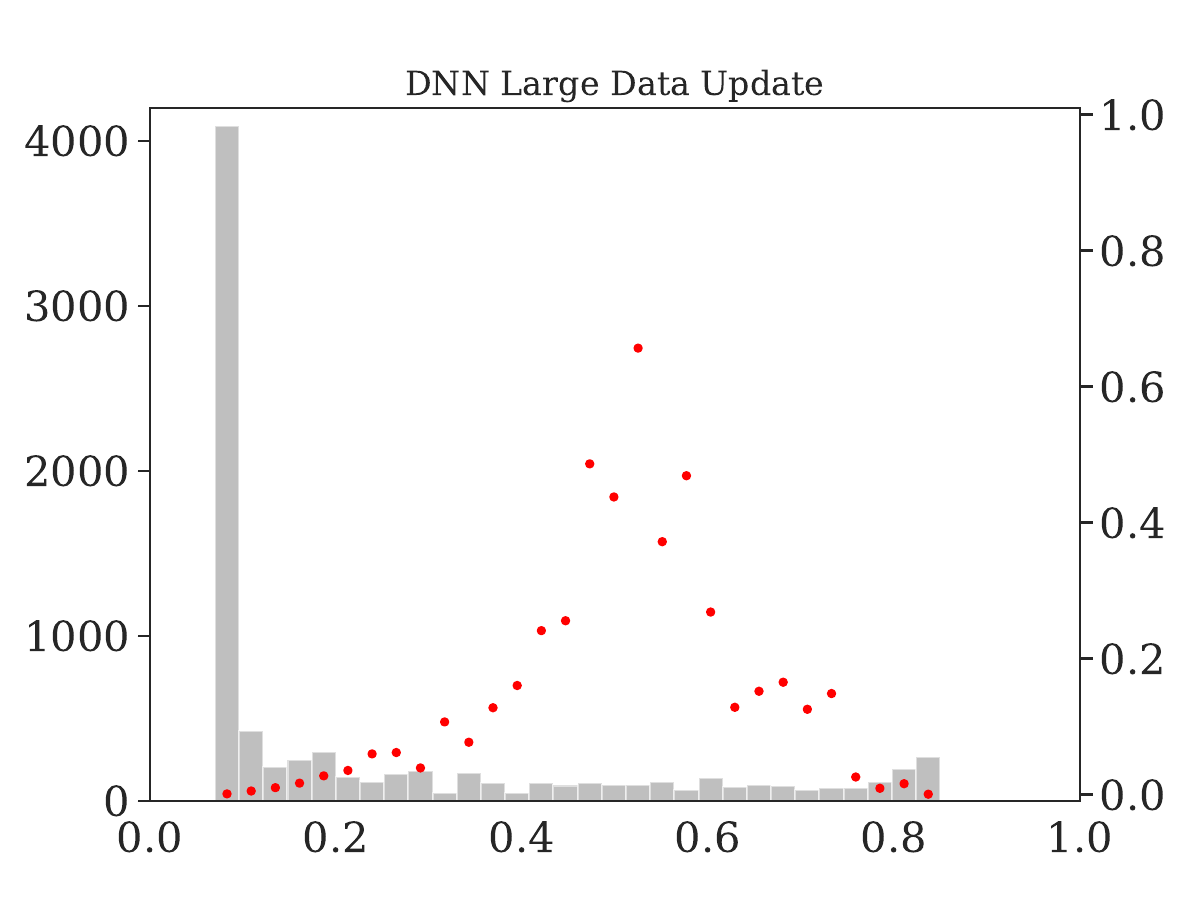} &
     \addPredProbsPlot{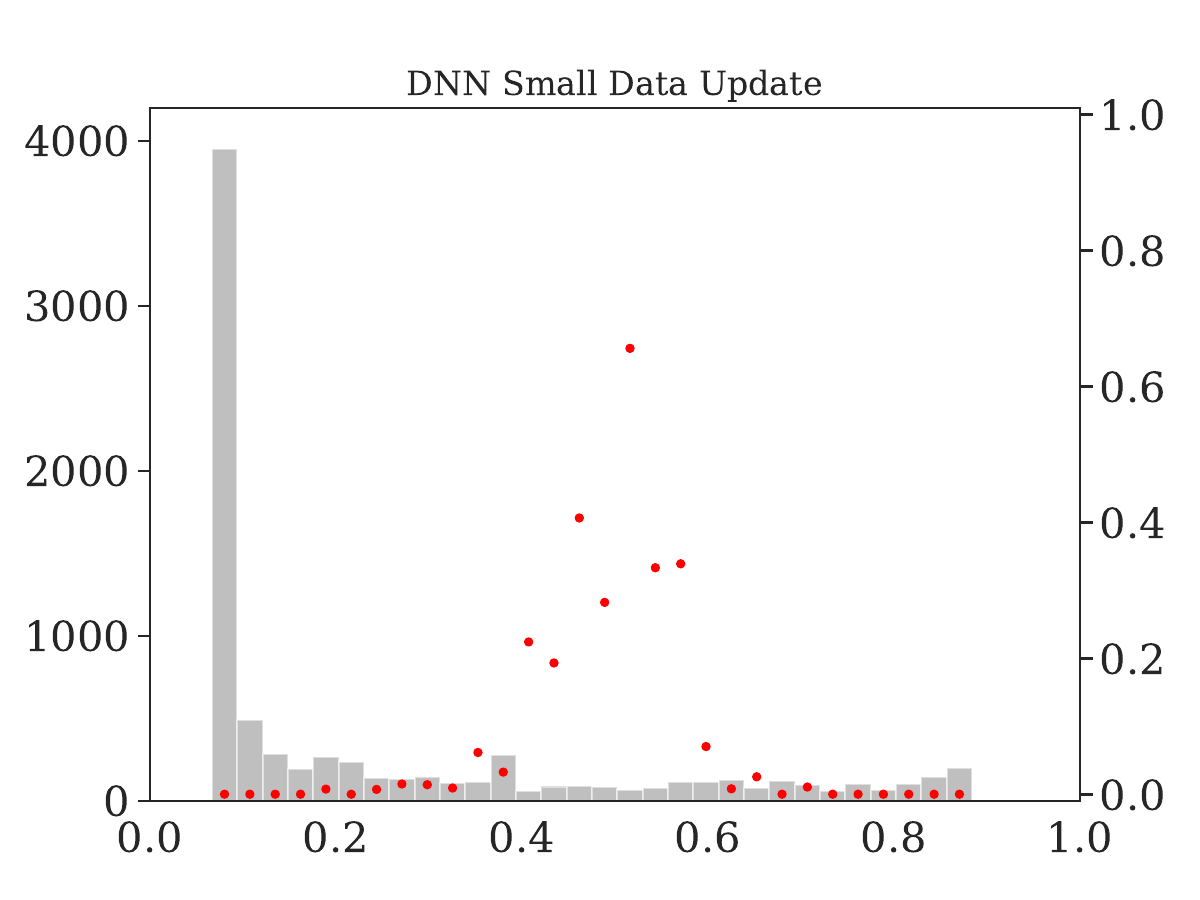}  \\
    \addPredProbsPlot{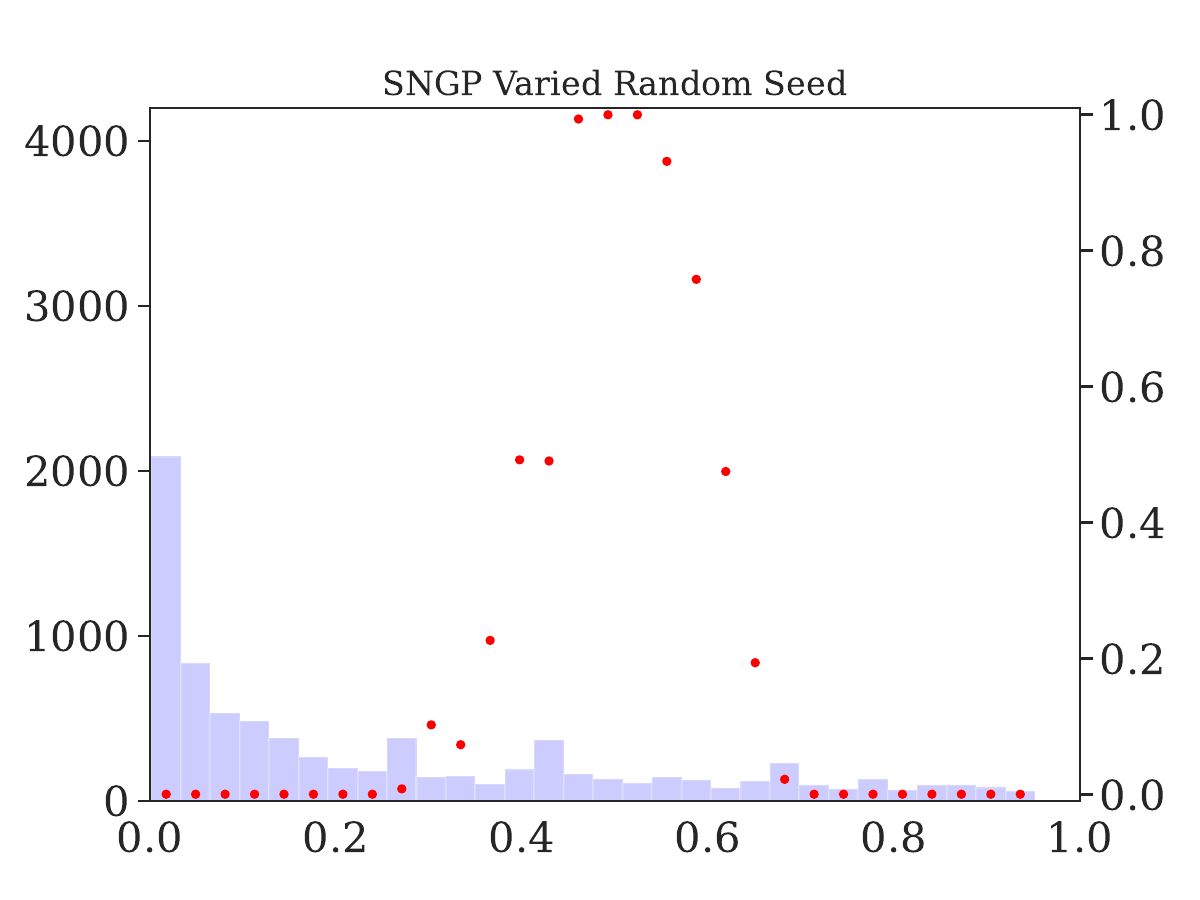} &
     \addPredProbsPlot{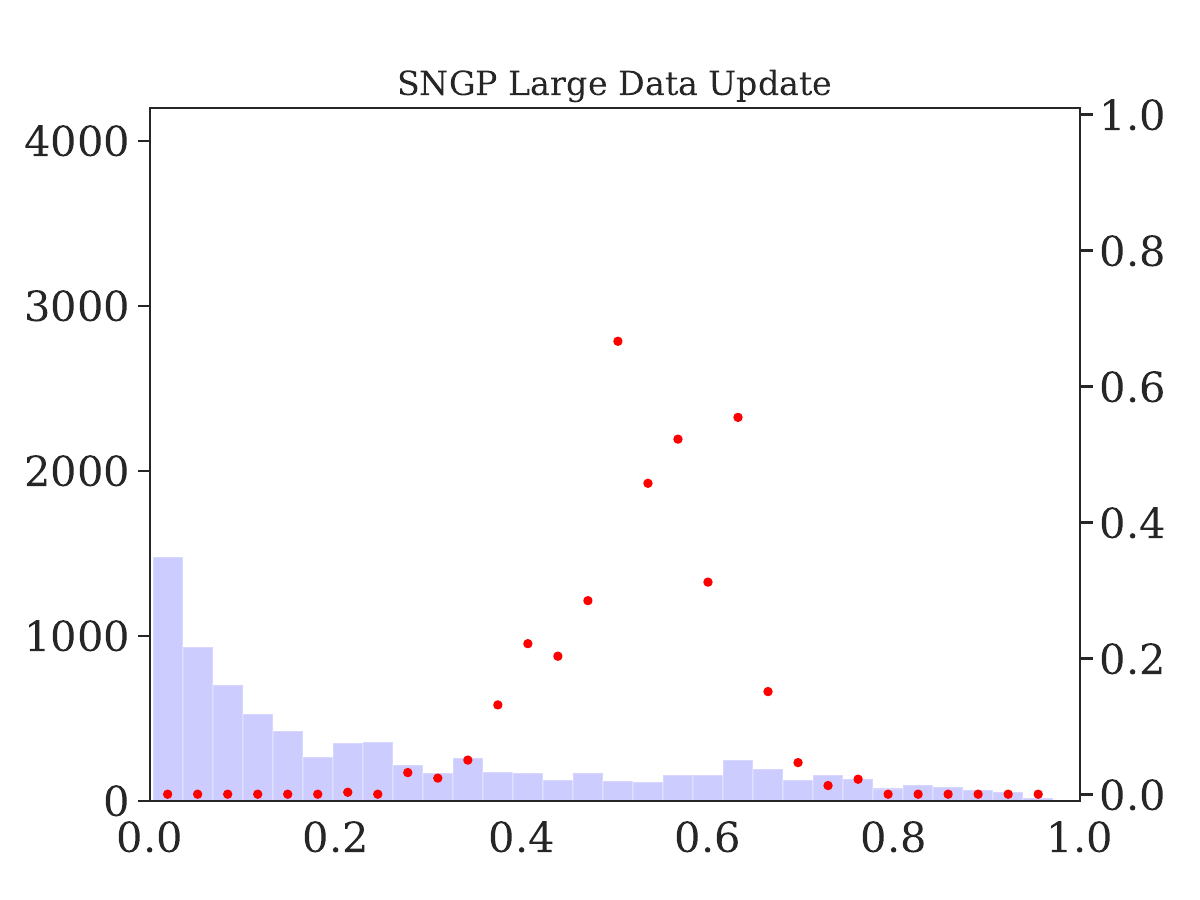} &
     \addPredProbsPlot{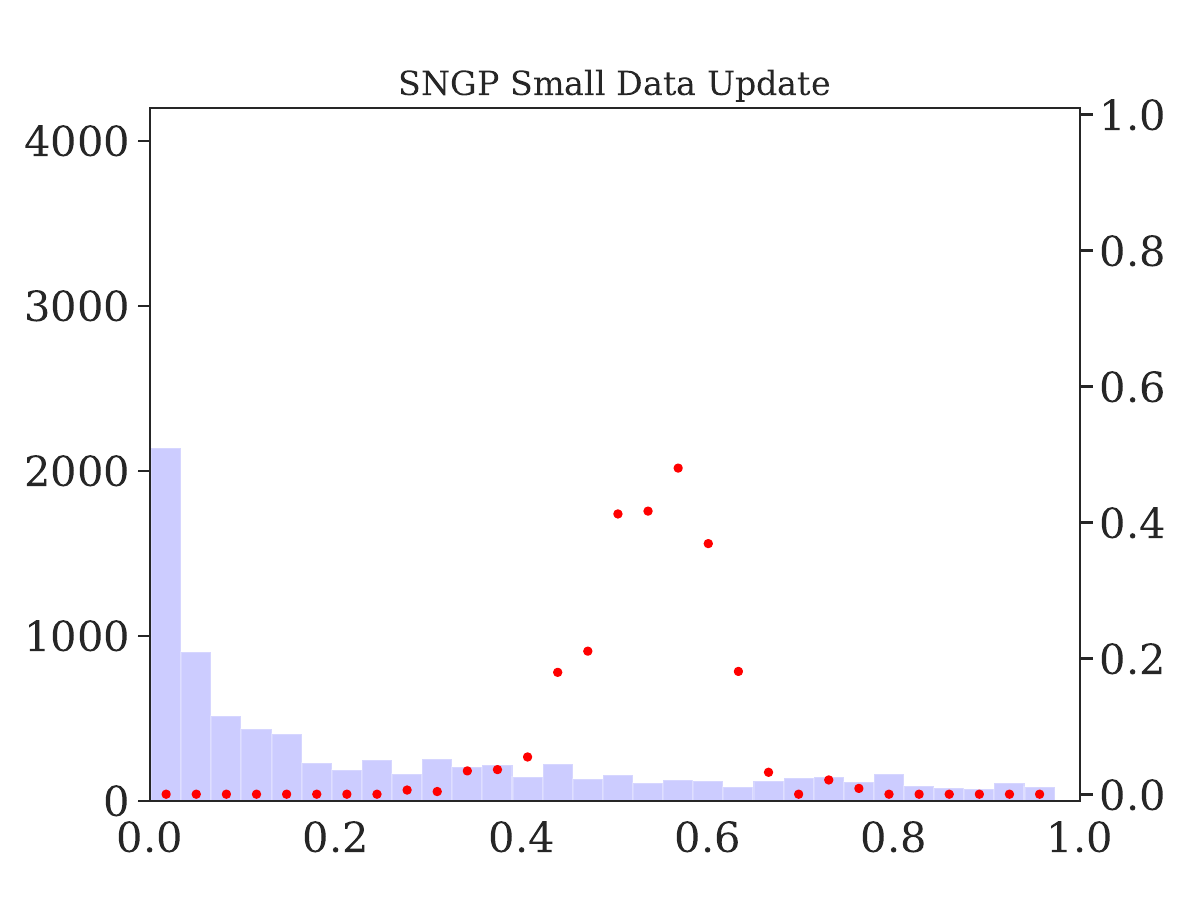} 
     \end{tabular}}
      \caption{Predicted probability distributions for the Adult Dataset. We plot a histogram of predicted probability distribution in grey with the left $y$-axis and a scatter plot of the proportion of flip counts for each bin aligned with the right $y$-axis. By overlapping the plots, we gain a comprehensive view of the model's confidence in its predictions (via the histogram) and the areas where the model predictions are most prone to change (scatter plot of flips). Notice that the scale is different between the histogram and the flip counts. The top row corresponds to the DNN experiments and the bottom row are the UA-DNN experiments. Each column represents an experiment. From the left, we show results for predictive multiplicity, large dataset update, and small dataset update.} \label{fig::pred_probs_uci}
\end{figure*}

\emph{Datasets.} We consider datasets with varying sample size, number of features, and class imbalance; summary statistics for each dataset are in Table~\ref{Table::Datasets}.~\footnote{Notice that the \textds{HMDA} dataset is an order of magnitude larger than the others ($n=244,107$).} As shown below, models trained and tested on these datasets exhibit notable variation in predictive inconsistency, i.e., this collection of datasets offers a reasonable variety.

\emph{Metrics.} We measure predictive inconsistency by computing the measures detailed in \S~\ref{sec::framework}. In terms of predictive multiplicity, we compute the empirical $\epsilon$-Rashomon set and report $\epsilon$-ambiguity over a test sample according to Eq~\eqref{eq::ambiguityRand}. As in \S~E.2 in \citet{long2023arbitrariness}, we can set $\epsilon$ in the definition of the empirical $\epsilon$-Rashomon set to the worst value of the performance metric over the generated trained models. As a result, the experiments on predictive multiplicity do not need to be explicitly parametrized by $\epsilon$. Regarding predictive churn, we report over a test sample according to Eq.~\eqref{eq::churn}.

\emph{Churn Regimes.} We compute \emph{predictive churn} Eq.~\eqref{eq::churn} for different types of successive training updates according to the literature on predictive churn~\citep{LaunchAndIterate}. First, we imitate a large dataset update by comparing Model $B$ ($\clf_B$) trained on the full dataset to Model $A$ ($\clf_A$) trained on a random sample of half the dataset. Second, we imitate a small dataset update by comparing Model $B$ ($\clf_B$) trained on the full dataset to Model $A$ ($\clf_A$) trained on a random sample of $95\%$ the dataset -- i.e., $5\%$ of examples have been dropped or added between the two models. These two updates are similar but represent two different regimes (see~\cite{pmlr-v89-giordano19a}).~\footnote{For instance, the leave-one-out jackknife is a small data perturbation, whereas bootstrap is a large data perturbation; see papers on infinitesimal jackknife, i.e., \citet{pmlr-v89-giordano19a}.} 

\emph{Model Classes.} We consider two classes of deep neural networks (DNNs). We train a standard neural network of 1 or more fully connected layers and refer to this as DNN. We also train a DNN that incorporates an uncertainty awareness technique, which we refer to as UA-DNN. For this demonstration, we implement the SNGP technique described in \S~\ref{sec::methods} to train the uncertainty-aware model, UA-DNN. To ensure the models are well calibrated, we tune the parameters within the SNGP technique and apply Platt scaling for the fully connected DNN—additional details in the appendix.

%%% Results Table
\begin{table*}[!t]
    \centering
    \resizebox{0.9\linewidth}{!}{\begin{tabular}{ll | cc | cc | cc }
    \toprule
    % & & & &
    % \multicolumn{2}{c}{Error of $\baseclf{}$} \\ 
    % \cmidrule(lr){5-6} 
    %
    \cell{l}{Dataset } &
    \cell{l}{Model } &
    \cell{l}{Predictive Multiplicity \\ (Empirical $\epsilon$-Ambiguity) } & AUC &
    \cell{l}{Predictive Churn \\ (Large Data Update)} & AUC &
    % \cell{l}{Predictive Churn \\ (Drop Random Feature)} & AUC & 
    \cell{l}{Predictive Churn  \\ (Small Data Update)} & AUC
    \\ 
    \toprule
         \textds{Adult} & DNN & 
         0.047 $\pm$ 0.003  & 0.89 $\pm$ 0.010 % amb
         & 0.058 $\pm$ 0.004 &  0.89 $\pm$ 0.009 % large
         & 0.028 $\pm$ 0.004 & 0.89 $\pm$ 0.01 \\ % small
         \textds{Credit} & DNN 
         & 0.053 $\pm$ 0.004  & 0.76 $\pm$ 0.01 % varied seed
         & 0.050 $\pm$ 0.004 & 0.76 $\pm$ 0.009 % large
         &  0.029 $\pm$ 0.004 & 0.76 $\pm$ 0.01 \\ % small
         \textds{HMDA} & DNN 
         & 0.021 $\pm$ 0.004 & 0.89 $\pm$ 0.011 % varied seed
         & 0.042 $\pm$ 0.004 & 0.89 $\pm$ 0.009 % large
         &  0.007 $\pm$ 0.004 &  0.89 $\pm$ 0.01\\ % small
        \textds{mammo} & DNN 
         & 0.007 $\pm$ 0.0018 & 0.83 $\pm$ 0.001 % varied seed
         & 0.027 $\pm$ 0.024 & 0.85 $\pm$ 0.007 % large
         &  0.014 $\pm$ 0.017 &  0.83 $\pm$ 0.004 \\ % small
         \midrule
         \textds{Adult} & UA-DNN 
         & 0.12 $\pm$ 0.010  & 0.87 $\pm$ 0.015  % varied seed
         & 0.074 $\pm$ 0.011 & 0.84 $\pm$ 0.012 % large data
         & 0.041 $\pm$ 0.008 & 0.87 $\pm$ 0.016 \\ % small ex
         \textds{Credit} & UA-DNN 
         & 0.10 $\pm$ 0.010 & 0.76 $\pm$ 0.015  % varied seed
         & 0.067 $\pm$ 0.012 &  0.76 $\pm$ 0.012 % large
         & 0.05 $\pm$ 0.008 & 0.76 $\pm$ 0.016 \\
         \textds{HMDA} & UA-DNN 
         & 0.14 $\pm$ 0.010 & 0.87 $\pm$ 0.015 % varied seed
         & 0.12 $\pm$ 0.011 & 0.84 $\pm$ 0.013 % large
         & 0.06 $\pm$ 0.008 &  0.87  $\pm$ 0.016 \\ % small
        \textds{mammo} & UA-DNN 
         & 0.047 $\pm$ 0.013 & 0.82 $\pm$ 0.001 % varied seed
         & 0.041 $\pm$ 0.019 & 0.83 $\pm$ 0.005 % large
         &  0.025 $\pm$ 0.020 &  0.83 $\pm$ 0.004 \\ % small
    \end{tabular}}
    \caption{This table shows that predictions are more sensitive to model perturbations (multiplicity) and an uncertainty-aware (UA) model can exhibit higher ambiguity compared to a standard DNN. We compare predictive multiplicity and predictive churn across datasets and model specifications. Over a held out sample $\sample_{\textit{test}}$, we compute empirical $\epsilon$-ambiguity, $\ambiguity{\epssetRand{}}$, as well as churn, $\churn{}(\clf_A, \clf_B)$, induced by a large or small data update. We also show the range of AUC over runs for each. 
    }\label{tab:main_results}
\end{table*}

\begin{table*}[!t]
    \centering
    \resizebox{0.9\linewidth}{!}{\begin{tabular}{ll | cc | cc | cc }
    \toprule
    % & & & &
    % \multicolumn{2}{c}{Error of $\baseclf{}$} \\ 
    % \cmidrule(lr){5-6} 
    %
    \cell{l}{Dataset } &
    \cell{l}{Model } &
    \cell{l}{Predictive Multiplicity \\ (Empirical $\epsilon$-Ambiguity) } & AUC &
    \cell{l}{Predictive Churn \\ (Large Data Update)} & AUC &
    % \cell{l}{Predictive Churn \\ (Drop Random Feature)} & AUC & 
    \cell{l}{Predictive Churn  \\ (Small Data Update)} & AUC
    \\ 
    \toprule
         \textds{Adult} & DNN & 
         0.004 $\pm$ 0.001  & 0.89 $\pm$ 0.001 % varied seed
         & 0.002 $\pm$ 0.006 &  0.89 $\pm$ 0.001 % large
         & 0.003 $\pm$ 0.001 & 0.89 $\pm$ 0.001 \\ % small
         \textds{Credit} & DNN 
         & 0.005 $\pm$ 0.0004  & 0.76 $\pm$ 0.002 % varied seed
         & 0.003 $\pm$ 0.0001 & 0.76 $\pm$ 0.004 % large
         &  0.0028 $\pm$ 0.0004 & 0.76 $\pm$ 0.002 \\ % small
         \textds{HMDA} & DNN 
         & 0.005 $\pm$ 0.001 & 0.90 $\pm$ 0.0003 % varied seed
         & 0.004 $\pm$ 0.001 & 0.90 $\pm$ 0.0004 % large
         &  0.003 $\pm$ 0.001 &  0.90 $\pm$ 0.0003\\ % small
        \textds{mammo} & DNN 
         & 0.004 $\pm$ 0.003 & 0.86 $\pm$ 0.003 % varied seed
         & 0.004 $\pm$ 0.003 & 0.85 $\pm$ 0.009 % large
         &  0.002 $\pm$ 0.002 &  0.85 $\pm$ 0.01 \\ % small
         \midrule
         \textds{Adult} & UA-DNN 
         & 0.0 $\pm$ 0.0  & 0.89 $\pm$ 0.002  % varied seed
         & 0.028 $\pm$ 0.0001 & 0.87 $\pm$ 0.002 % large data
         & 0.019 $\pm$ 0.002 & 0.88 $\pm$ 0.003 \\ % small ex
         \textds{Credit} & UA-DNN 
         & 0.0 $\pm$ 0.0 & 0.75 $\pm$ 0.004  % varied seed
         & 0.035 $\pm$ 0.003 &  0.75 $\pm$ 0.006 % large
         & 0.020 $\pm$ 0.002 & 0.75 $\pm$ 0.003 \\
         \textds{HMDA} & UA-DNN 
         & 0.0 $\pm$ 0.0 & 0.90 $\pm$ 0.001 % varied seed
         & 0.046 $\pm$ 0.002 & 0.90 $\pm$ 0.0001 % large
         & 0.041 $\pm$ 0.002 &  0.90  $\pm$ 0.0002 \\ % small
        \textds{mammo} & UA-DNN 
         & 0.0 $\pm$ 0.0 & 0.84 $\pm$ 0.003 % varied seed
         & 0.02 $\pm$ 0.009 & 0.83 $\pm$ 0.010 % large
         &  0.005 $\pm$ 0.006 &  0.84 $\pm$ 0.008 \\ % small
    \end{tabular}}
    \caption{Ensemble Results. We compare predictive multiplicity and predictive churn across datasets and model specifications. Over a held out sample $\sample_{\textit{test}}$, we compute empirical $\epsilon$-ambiguity, $\ambiguity{\epssetRand{}}$, as well as churn, $\churn{}(\clf_A, \clf_B)$, induced by a large or small data update. We also show the range of AUC over runs for each. 
    }\label{tab:ensemble_results}
\end{table*}

\subsection{Results}

\paragraph{Predictive Multiplicity vs Predictive Churn.}
We investigate whether the severity of predictive churn between Model $A$ and Model $B$ is captured by predictive multiplicity analysis on only Model $A$. Findings for the Standard DNN and UA-DNN are shown in Table~\ref{tab:main_results}.
Notably, we see that model performance, as measured by AUC, is mostly uniform across the table: random seed/data perturbations (columns) do not seem to affect overall predictive performance, whereas AUC of the UA-DNN is less than or equal to that of DNN. 

We highlight several patterns.
First, although they are measured on similar scales, predictive multiplicity (ambiguity) tends to be larger than predictive churn. Thus, in the settings that we study, predictions appear to be broadly more sensitive to model perturbations than to data updates. But only by a small amount. 

Second, within model specifications (DNN or UA-DNN), predictive multiplicity and predictive churn measurements generally align, i.e., high predictive multiplicity corresponds to high predictive churn (across both churn regimes) relative to other datasets. Thus, for a given model, it is possible that the same properties of the dataset drive predictive multiplicity and predictive churn.

However, interestingly, between the DNN and UA-DNN, we see that different datasets exhibit high prediction inconsistencies. For example, while DNN exhibits high(er) predictive multiplicity on $\textds{Credit}$, the UA-DNN exhibits higher ambiguity on $\textds{HMDA}$ but relatively lower on $\textds{Credit}$. This highlights that prediction inconsistency is driven by an interaction between the dataset and the model specification, not by the data itself, echoing predictive arbitrariness studies from algorithmic fairness~\cite{cooper2023prediction}. This also highlights that a particular model spec may not be a general solution for mitigation.

\paragraph{Comparison of Unstable Sets.}
We examine whether examples that are unstable over the update between Model $A$ and Model $B$ are included in those flagged as unstable when only using the $\epsilon$-Rashomon set of Model $A$. For a given dataset, we take a heldout test sample and compute $\unstablesetR{}(\sample_{\textit{test}})$ and $\unstablesetC{}(\sample_{\textit{test}})$. Given that $\#\{\unstablesetR{}(\sample_{\textit{test}})\}$ tends to be greater than $\#\{\unstablesetC{}(\sample_{\textit{test}})\}$, we calculate what proportion of test examples in $\#\{\unstablesetC{}(\sample_{\textit{test}})\}$ are contained in $\#\{\unstablesetR{}(\sample_{\textit{test}})\}$ and report this common inconsistency. 

For instance, if all the examples in $\sample_{\textit{test}}$ that churn are contained in the $\epsilon$-Rashomon unstable set, then the common inconsistency would be $100\%$. As expected, for the small data updates, the common inconsistency is much higher than compared to the large data update. Comparing model classes, the UA-DNN for small dataset updates seems to recover the most significant overlap (results in app. Table~\ref{tab:flip_idx}).

\paragraph{Predicted Probabilities and Unstable Examples.}
Finally, we examine how predicted probabilities relate to which points are identified as unstable. With the $\epsilon$-Rashomon unstable set and the churn unstable sets over a given test sample, we visualize the number of unstable examples alongside the full predicted probability distribution in Figure~\ref{fig::pred_probs_uci}. First, we plot a histogram of the predicted probabilities for the test sample. Then, for each bin of the histogram, we compute the counts of the unstable (flipped) examples within that bin. Namely, the number of unstable (flipped) examples in a bin divided by the total number of predictions in that bin. This highlights where the model's predictions are most unstable, as indicated by a higher proportion of unstable points.

% We see that predicted probabilities of flipped examples (red points) are similarly concentrated in the middle of the unit interval comparing DNN to UA-DNN, which is somewhat surprising given the explicit consideration of uncertainty in UA-DNN. But one side effect of this consideration is that small perturbations may send UA-DNN predictions across the default decision boundary, which could explain the generally higher rates of inconsistency in Table~\ref{tab:main_results}, especially under the predictive multiplicity perspective. The findings suggests that UA-DNN models can provide useful indications of which examples are more at risk of being unstable under perturbations of the UA-DNN model, as a result of both predictive multiplicity or churn. Hence, the results show that model specification may not be the driving factor here. The predicted probabilities around the threshold ($0.5$) are more likely to be unstable. Therefore, the important difference in model type seems to be calibration.

\paragraph{Ambiguity and Churn for Ensemble Classifiers.}
Given that ensembling decreases ambiguity~\cite{Black2021, long2023arbitrariness}, we compute ambiguity and churn for ensemble classifiers, showing results in Table~\ref{tab:ensemble_results}. Notably, the ambiguity for the uncertainty-aware model is zero across datasets. Moreover, churn has decreased significantly as well. These results support the intuition that predictive multiplicity reduction is related to churn reduction and that both perspectives might benefit from engaging with uncertainty-aware model types.

\paragraph{Predicting Churn.}
As described in \S~\ref{sec::methods}, we can train a classifier to predict churn to examine correlation between ambiguity and predictive churn. First, we examine the correlation between variables by analyzing the Pearson Correlation between the features, predicted probabilities, ambiguity indicator and churn indicator. We focus on correlation between ambiguity and churn. In Figure~\ref{fig::corr_DNN}, there is not much correlation between ambiguity and churn for the mammo and adult datasets (top left and right). But there does seem to be a negative correlation for the hmda and credit datasets (bottom left and right). This illuminates an interesting relationship between the two concepts.

\subsection{Implications}
\label{sec::implications}

Our findings reveal that analyzing predictive multiplicity is a useful way to anticipate predictive churn over time. We can consider the set of \emph{prospective models} around the selected deployed model and draw conclusions about anticipated predictive churn. Given that research in predictive multiplicity has largely focused on how to measure its severity and methods to train the $\epsilon$-Rashomon set, the present study demonstrates how predictive multiplicity can help assess an important notion of predictive instability (churn).

To combine predictive multiplicity and churn, a practitioner could conduct one analysis after the other. For choosing a better starting point while anticipating model updates, we can begin with a predictive multiplicity analysis following by a predictive churn analysis. Say for instance, we have a model $A$ that we are considering for deployment. We can ask if there might exist a model within the $\epsilon$-Rashomon set for which the anticipated churn 
is likely less than that of   model $A$. To do this, we can train the $\epsilon$-Rashomon set with model $A$ as a baseline then evaluate changes in the churn unstable set for each model within the Rashomon set. We can also train the $\epsilon$-Rashomon set without assuming a baseline and choose the model that might minimize expected churn from that.

Previous studies have examined various churn reduction methods~\cite{LaunchAndIterate, jiang2022churn}. It will be interesting in future work to examine whether known churn reduction methods (e.g., distillation and constrained weight optimization) might improve predictive multiplicity. To do this, we would analyze predictive multiplicity over a standard training procedure then, make improvements to said training procedure that for churn reduction and analyze predictive multiplicity over this improved training procedure. Similar to our empirical demonstrations, you can then take a fixed test set and compare the $\epsilon$-Rashomon unstable set against the churn unstable set. Ultimately, this would provide insight into whether training procedures that are more robust to churn are also more robust to predictive multiplicity. And, in line with bridging between uncertainty quantification and fairness as arbitrariness, future work can also explore additional methods from reliable deep learning i.e ~\cite{plex}.

\subsection{Limitations and Future Work}
\label{sec::limits}

While our theoretical results offer valuable insights, they are not without limitations. The $\beta$-stability assumption and smooth churn assumptions offer a convenient way to derive bounds, but the practical impact depends on the nature of the dataset and training procedure. While the assumptions are theoretically sound, they may not hold in all empirical cases. Also, the bounds derived with respect to a baseline are derived assuming that models are empirical risk minimizers, which would not work for optimization procedures that do not strictly minimize empirical risk. Moreover, the analytical upper bound on churn is helpful but may be overly conservative in practical settings. Despite these limitations, the bounds remain useful, offering a worst-case scenario that can give practitioners intuition on leveraging predictive multiplicity to anticipate risks associated with model updates. Future work could refine these bounds by relaxing the $\beta$-stability assumption or studying the tightness of the bounds in relation to empirical observations. Given that the empirical Rashomon set used in our experiments lacks a clear baseline, this is beyond the scope of the present study and better suited for future research.

The experiments in this study are valuable as they reveal interesting connections between predictive multiplicity and predictive churn. However, it is essential to acknowledge some limitations. As noted in our experiments, the Rashomon set is defined via the empirical Rashomon set without a clear baseline, which does not encompass the full range of Rashomon set definitions that one might consider while performing this analysis. For example, \S~\ref{sec::add-defs} shows an algorithm for training the Rashomon set with respect to a clear baseline. Incorporating such definitions would require extensive experimentation outside the scope of this paper but is primed for future work. Also, our experiments focus only on two model types (DNN and DNN-UA), potentially leaving an opportunity for a study entirely focused on a variety of model types and how they relate to these phenomena. Additionally, while our experiments demonstrate that a method aimed at reducing predictive multiplicity also reduces churn, our study does not explore the reverse scenario. Examining this reciprocal relationship could provide deeper insight and is a promising direction for future research.

\subsection{Concluding Remarks}
Understanding predictive inconsistency is crucial for both deploying ML in industry and addressing algorithmic fairness concerns. In this paper, we have taken initial steps to link two previously unconnected concepts: predictive churn and predictive multiplicity. Our work provides both theoretical and empirical insights, offering a foundation for further exploration.~\footnote{Please see \S~\ref{sec::implications} for a discussion of general implications, see the discussion in \S~\ref{sec::theory} for takeaways of theoretical results and see \S~\ref{sec::limits} for a more detailed limitations/future work discussion.} More broadly, we advocate for integrating research on fairness and safety with efforts to achieve reliable and robust learning as an opportunity to enhance the synergy between these fields. 

\section{Ethics and Adverse Impact Statement}
The study of predictive inconsistency is socially relevant to how people interpret and trust the output of predictive models. A better understanding of predictive inconsistency supports the ethical obligation (on the part of researchers and developers) to be transparent about models that are being used in real world settings. Our methodologies and discussions are mindful of these considerations.

\clearpage
\bibliographystyle{ACM-Reference-Format}
\bibliography{references.bib}

%%%%%%%%%%%%%%%%%%%%%%%%%%%%%%%%%%%%%%%%%%%%%%%%%%%%%%%%%%%%%%%%%%%%%%%%%%%%%%%
%%%%%%%%%%%%%%%%%%%%%%%%%%%%%%%%%%%%%%%%%%%%%%%%%%%%%%%%%%%%%%%%%%%%%%%%%%%%%%%
% APPENDIX
%%%%%%%%%%%%%%%%%%%%%%%%%%%%%%%%%%%%%%%%%%%%%%%%%%%%%%%%%%%%%%%%%%%%%%%%%%%%%%%
%%%%%%%%%%%%%%%%%%%%%%%%%%%%%%%%%%%%%%%%%%%%%%%%%%%%%%%%%%%%%%%%%%%%%%%%%%%%%%%
\newpage
\appendix
\onecolumn

\section{Omitted Proofs}

%%%% PROOF 1
\begin{proof}[Proof of Proposition~\ref{Prop::churn_bound}]
  This follows from the triangle inequality.
  %to bound the distance between predictions of two classifiers $\clf_A$ and  $\clf_B$.
  For a set $S=\{\xb_1,\ldots,\xb_n\}$,
 we denote the predictions as  vectors: 
    $$Y_1 = (\clf_1(\xb_1),...,\clf_1(\xb_n) ) \in \{0,1\}^n$$
    $$Y_2 = (\clf_2(\xb_1),...,\clf_2(\xb_n) ) \in \{0,1\}^n$$
    Let $Y$ denote the ground-truth label,
    $$Y = (y_1,...,y_n) \in \{0,1\}^n.$$

    The empirical risk $\erm{}$
    of a classifier can be expressed in terms of the $L_1$ norm between the predictions 
and the ground truth: 
    $$ \erm{}(\clf_1) = \frac{||Y_1 - Y||_1}{n}, \quad  \erm{}(\clf_2) = \frac{||Y_2 - Y||_1}{n} $$
    
    Similarly, we write churn as the $L_1$ norm between the predictions of the two models. 

    $$ \churn{}(\clf_1,\clf_2) = \frac{||Y_1 - Y_2||_1}{n} $$

    The triangle inequality results in:

    $$ ||Y_1-Y_2||_1 \leq ||Y_1-Y||_1 + ||Y-Y_2||_1 $$

  Substitution and dividing by $n$ gives
    $$\churn{}(\clf_1,\clf_2) \leq \erm{}(\clf_1) + \erm{}(\clf_2).$$
\end{proof}

\begin{proof}[Proof of Corollary~\ref{cor::churn-bound}]
    By definition, $\erm(\clf{}') \leq \erm(\clf{}_0)+ \epsilon$. Following Proposition \ref{Prop::churn_bound}, we have:
    \begin{align}
        \!  \churn{} (\clf{}_0, \clf{}') \!  &\leq \erm{}(\clf_0) + \erm{}(\clf{}')
  \leq  2 \erm{}(\clf{}_0) + \epsilon.
    \end{align}
\end{proof}

%%%% PROOF 2
\begin{proof}[Proof of Lemma~\ref{lem::churn-diff}]
We use linearity of expectation and the assumption that models in $\mathcal{T}_{\data_A}$ are sampled i.i.d.~to show that the difference in expectation is $0$.
\begin{align*}
     &\mathbb{E}_{\clf{}_A, \clf{}'_A \stackrel{iid}{\sim} \trainRand{}(\data{}_A)} \left[\churn{}(\clf{}_A, \clf{}_B) - \churn{}(\clf{}'_A, \clf{}_B) \right] \\
     &=\mathbb{E}_{\clf{}_A \stackrel{iid}{\sim} \trainRand{}(\data{}_A)} [\churn{}(\clf{}_A, \clf{}_B)] \\  &\quad - \mathbb{E}_{\clf{}'_A \stackrel{iid}{\sim} \trainRand{}(\data{}_A)} [\churn{}(\clf{}'_A, \clf{}_B))] \\
    &=\mathbb{E}_{\clf{}_A \stackrel{iid}{\sim} \trainRand{}(\data{}_A)} [\mathbb{E}_{(X,Y)\sim \data{}} [\ell_{0,1}(\clf{}_A(X),Y) - \ell_{0,1}(\clf{}_B)(X),Y)]] \\  &\quad - \mathbb{E}_{\clf{}'_A \stackrel{iid}{\sim} \trainRand{}(\data{}_A)} [\mathbb{E}_{(X,Y)\sim \data{}} [\ell_{0,1}(\clf{}'_A(X),Y) - \ell_{0,1}(\clf{}_B)(X),Y)]] \\
    &=\mathbb{E}_{\clf{}_A, \clf{}'_A \stackrel{iid}{\sim} \trainRand{}(\data{}_A)} [\mathbb{E}_{(X,Y)\sim \data{}} [\ell_{0,1}(\clf{}_A(X),Y) - \ell_{0,1}(\clf{}'_A(X),Y)]] \\
     &=\mathbb{E}_{\clf{}_A \stackrel{iid}{\sim} \trainRand{}(\data{}_A)} [\mathbb{E}_{(X,Y)\sim \data{}} [\ell_{0,1}(\clf{}_A(X),Y)]] - \mathbb{E}_{\clf{}'_A \stackrel{iid}{\sim} \trainRand{}(\data{}_A)}[\mathbb{E}_{(X,Y)\sim \mathcal{D}} [\ell_{0,1}(\clf{}'_A(X),Y)]] \\
     &= 0.
\end{align*}
\end{proof}

%%%% PROOF 3
\begin{proof}[Proof of Theorem~\ref{thm::churn-rashomon}]

We first state the results from ~\citet{LaunchAndIterate}.
\begin{theorem}[Bound on Expected Churn~\citep{LaunchAndIterate}] Assuming a training algorithm that is $\beta$-stable, given training datasets $\data_A$ and $\data_B$, sampled i.i.d.~from $\data{}^n$ where two classifiers $\clf{}_A$ and $\clf{}_B$ are trained on $\data_A$ and $\data_B$ respectively, the expected smooth churn obeys:
\begin{align}
    \mathbbm{E}_{\data_A, \data_B \sim \data^n} \left[ \churn_{\gamma}(\clf{}_A, \clf{}_B) \right] \leq \frac{\beta\sqrt{\pi n}}{\gamma}.
\end{align}
\label{thm::expected_churn_beta_stable}

\end{theorem}

From Theorem~\ref{thm::expected_churn_beta_stable}, the smooth churn between the two baseline models is bounded by:
\begin{align*}
    \mathbbm{E}_{\data_A, \data_B \sim \data^m} [\churn{}_{\gamma}(\clf{}_0^A, \clf{}_0^B)] \leq \frac{\beta\sqrt{\pi n}}{\gamma}.
\end{align*}

% \dcp{oh, are the models in thm 5.4 optimal on training data (or generalization error?) and is
% this required for these baselines here?} \carol{this is a bound on generalization error we cited in Theorem 5.5. there's no assumption on baseline model, only assumption is that the training algorithm is $\beta$-stable} 

The churn between any two models within the $\epsilon$-Rashomon sets, $\epsset{\clf{}_0^A}$ and $\epsset{\clf{}_0^B}$, is bounded by this constant plus a new $2\epsilon$ term.
To show this, we  apply the triangle inequality and Lemma~\ref{lem::churn-diff}, working
with any pair of models, $\clf{}'_A \in \epsset{\clf_0^A}$ and $\clf{}'_B \in \epsset{\clf{}_0^B}$:
\allowdisplaybreaks
\begin{align*}
    &\mathbbm{E}_{\data_A, \data_B \sim \data^m} [\churn{}_{\gamma}(\clf{}'_A, \clf{}'_B)] \\
    &= \mathbbm{E}_{(X, Y) \sim \data} \left[ \ell_{\gamma}(\clf{}'_A(X), Y) - \ell_{\gamma}(\clf{}'_B, Y) \right] \\
    &= \mathbbm{E}_{(X, Y) \sim \data} [ \ell_{\gamma}(\clf{}'_A(X), Y) + \ell_{\gamma}(\clf{}^A_0(X), Y)  \\ &- \ell_{\gamma}(\clf{}^A_0(X), Y) + \ell_{\gamma}(\clf{}^B_0(X), Y) - \ell_{\gamma}(\clf{}^B_0(X), Y) \\ &- \ell_{\gamma}(\clf{}'_B, Y) ] \\
    &= \mathbbm{E}_{(X, Y) \sim \data} \left[ \ell_{\gamma}(\clf{}'_A(X), Y) -  \ell_{\gamma}(\clf{}^A_0(X), Y) \right] \\
    &+ \mathbbm{E}_{(X, Y) \sim \data} \left[ \ell_{\gamma}(\clf{}^A_0(X), Y) - \ell_{\gamma}(\clf{}^B_0(X), Y) \right] \\
    &+ \mathbbm{E}_{(X, Y) \sim \data} \left[ \ell_{\gamma}(\clf{}^B_0(X), Y)  - \ell_{\gamma}(\clf{}'_B, Y) \right] \\
    & \leq \epsilon + \frac{\beta\sqrt{\pi n}}{\gamma} + \epsilon = \frac{\beta\sqrt{\pi n}}{\gamma} + 2 \epsilon,
\end{align*}
where the second and third equalities are algebra.
For the inequality, the first and third expectations  
follow from the Definition of smooth churn and the middle expectation from 
 Theorem \ref{thm::expected_churn_beta_stable}.
 For the final equality,  we 
 appeal to
 Definition~\ref{def::rashomon}, with $\ell_\gamma$ as the performance metric and $\epsilon$ being the parameter of the Rashomon set. 
\end{proof}

\clearpage

\section{Additional Experimental Details}

\paragraph{Models} All models use a shallow neural network with 1 or more fully connected layers. There is 1 hidden layer with 279 units, learning rate of 0.0000579, dropout rate of 0.0923 and batch normalization is enabled. All training is conducted in TensorFlow with a batch size of 128. When training sets of models, we use multiple arrays of random seeds \{\scalebox{0.7}{0.0}, \scalebox{0.7}{1.0}, \scalebox{0.7}{109}, \scalebox{0.7}{10}, \scalebox{0.7}{1234}\}, 
\{\scalebox{0.7}{3666}, \scalebox{0.7}{2299}, \scalebox{0.7}{2724}, \scalebox{0.7}{1262}, \scalebox{0.7}{4220}\}, 
\{\scalebox{0.7}{3971}, \scalebox{0.7}{9444}, \scalebox{0.7}{1375}, \scalebox{0.7}{7351}, \scalebox{0.7}{2083}\}, 
\{\scalebox{0.7}{1429}, \scalebox{0.7}{2281}, \scalebox{0.7}{2189}, \scalebox{0.7}{9376}, \scalebox{0.7}{2261}\} 
and \{\scalebox{0.7}{1881}, \scalebox{0.7}{2273}, \scalebox{0.7}{9509}, \scalebox{0.7}{6707}, \scalebox{0.7}{4412}\}. For varying random initialisations, we repeat experiments across these arrays. For churn experiments, we use the first random seed in the array as the default seed and repeat experiments across these values. We run on a single CPU with 50GB RAM. 

The SNGP training process follows the standard DNN learning pipeline, with the updated Gaussian process and spectral normalization  outputting  predictive logits and posterior covariance. 
%The steps for SNGP prediction are as follows. 
For a  test example, the model posterior mean and covariance are used to compute the predictive distribution.
Specifically, we approximate the posterior predictive probability, $E(p(x))$,
using the mean-field method
$ E(p(x)) \sim \text{softmax} \left( \text{logit}(x)/ \sqrt{1 + \lambda * \sigma^2(x)}\right)$,
where $\sigma^2(x)$ is the SNGP variance and $\lambda$ is 
a  hyperparameter, tuned for optimal model calibration (in deep learning,
this is known as temperature scaling~\citep{pmlr-v70-guo17a}).

\clearpage

\section{Additional Definitions}
\label{sec::add-defs}

% \paragraph{Beta-Stability}
% Below, we define $\beta$-stability which is used in smooth churn theoretical assumptions.

\paragraph{Predictive Multiplicity}
As an example of a training procedure that approximates the empirical $\epsilon$-Rashomon set with respect to a baseline model, we review the following. As noted in the main paper, these two metrics for quantifying predictive multiplicity reflect the proportion of examples in a sample $S$
that are assigned conflicts (or ``flips") over the $\epsilon$-Rashomon set. 
\begin{definition}[$\epsilon$-Ambiguity w.r.t. $\clf_0$]
The {\em ambiguity} of a prediction problem w.r.t.~$\clf{}_0$
is the proportion
of examples $i \in S$ assigned a conflicting prediction by a 
classifier in the $\epsilon$-Rashomon set:

\begin{align}
  \!  \ambiguity{\clf{}_0} \! :=\! \frac{1}{|S|} \sum_{i\in S} \max_{h \in \epsset{\clf{}_0}} \indic{h(x_i) \neq h_0(x_i)}. \label{eq::ambiguityBASE}
\end{align}
\end{definition}

\begin{definition}[Discrepancy w.r.t. $\clf_0$]
The {\em discrepancy} of a prediction problem  w.r.t.~$\clf{}_0$
is the \underline{maximum} proportion of examples $i \in S$ assigned a conflicting prediction by a single
competing classifier in the $\epsilon$-Rashomon set:

\begin{align}
  \!  \discrepancy{\clf{}_0} \! :=\!   \max_{h \in \epsset{\clf{}_0}} \frac{1}{|S|} \sum_{i\in S} \indic{h(x_i) \neq h_0(x_i)}.  \label{eq::discrepancy}
\end{align}
\end{definition}

Ambiguity characterizes the number of individuals whose predictions are sensitive to model choice with respect to the set of near-optimal models.
In domains where predictions inform decisions (e.g., loan approval or recidivism risk), 
individuals with ambiguous decisions could contest the prediction assigned to them. In contrast, discrepancy measures the maximum number of predictions that can change by replacing the baseline model 
with another near-optimal model. 

% Adversarial Training Approximate Rashomon Set
%For settings with a baseline, researchers 

An approach to compute these metrics is to 
approximate the Rashomon set by directly perturbing the
target loss in training~\cite{Marx2019, Watson-Daniels2022, multitarget}. 
We denote this {\em loss-targeting method} as $\trainAdv{}(\clf{}_0, \data{})\subseteq \Hset$, and it returns a set of hypotheses in the $\epsilon$-Rashomon set.  %\dcp{consider ${\mathcal T}_{\mathrm{target}}$}
For shorthand notation, we  leave implicit in the sequel 
the baseline and  dataset 
  in  notation $\trainAdv{}$.
\begin{definition}[Empirical $\epsilon$-Rashomon set w.r.t. $\clf{}_0$]
Given a performance metric $M$, an error tolerance $\epsilon >0$, and a baseline model $\clf{}_0$, the {\em empirical $\epsilon$-Rashomon set w.r.t.~$\clf{}_0$}
 is the set of competing classifiers $\clf \in \Hset$ induced by $\trainAdv{}$:

\allowdisplaybreaks
\begin{align}
 \epssetAdv{}(\trainAdv{}) \! :=\!  \{\clf{}\,:\, \clf{}\in \trainAdv{}, M(\clf{}; \data{}) \leq M(\clf{}_0; \data{})  +  \epsilon \}.
\end{align}
\end{definition}

\paragraph{An example of $\trainAdv{}$} Here is an example of $\trainAdv{}$.
% \section{$\epsilon$-Rashomon set for linear models}
% \label{Sec::probclf}
\citet{Watson-Daniels2022} introduced a method for computing ambiguity and discrepancy that involves training the Rashomon set as follows. A set of candidate models are trained via constrained optimization such that $P(\hat{y}_i=+1)$ is constrained to the threshold $p$ as in Eq.~\eqref{Def::candidates}. From that set of candidate models, those with near-optimal performance are selected.

\begin{algorithm}[h]
\begin{algorithmic}[1]\small
\STATE {\bfseries Input:} data $\{(x_i, y_i)\}_{i=1}^{n}$ 
\STATE {\bfseries Input:} baseline model $\{h_0{(x_i)}\}_{i=1}^n$
\STATE {\bfseries Input:} threshold probabilities $P$
\STATE {\bfseries Input:} error tolerance $\epsilon$
\FOR{$i \in \{(x_i, y_i)\}_{i=1}^{n}$}
\STATE Initialize $x_p, y_p = X(i), Y(i)$
\FOR{$p \in P$}
\STATE $h \gets$ model from Eq.~\eqref{Def::candidates}
%\STATE $\{\text{loss, CAL, AUC}\} \gets$ metrics for $\pathclf$
\STATE $pr(x_p) \gets$ $h{(x_p)}$
\ENDFOR
\ENDFOR
\STATE candidate models $\in \{h, pr(x_p) \}_{i\in [n], p \in P}$ 
\STATE $\epsilon$-Rashomon set $\gets$ candidate models that perform within $\epsilon$ of $h_0$
\STATE {\bfseries Output:} $\epsilon$-Rashomon set
\end{algorithmic}
\caption{Constructing the $\epsilon$-Rashomon set
}\label{Alg::rashomonSet}
\end{algorithm}

\newcommand{\lossfun}[1]{L(#1)}
\begin{definition}[Candidate Model]
Given a baseline model $h_0$, a finite set of user-specified threshold probabilities $P \subseteq [0,1]$, then for each $p\in P$ a {\em candidate model} for example $x_i$ is an optimal solution to the following constrained empirical risk minimization problem:
\begin{align}
\begin{split}
\min_{w \in \R^{d+1}} & \ \lossfun{w} \\
\st  \quad & h(x_i) \leq p, \;\; \textrm{if}~ p < h_0(x_i)\\
& h(x_i) \geq p.  \;\;\; \textrm{if} ~ p > h_0(x_i)
\end{split}\label{Def::candidates}
\end{align}
\end{definition}

This technique can be applied to any convex loss function $L(\cdot)$ including a convex regularization term. \citet{Watson-Daniels2022} illustrate the methodology on a probabilistic classification task with logistic regression where $h(x_i) = \frac{1}{1 + \exp(- \langle w, x_i \rangle )}$.

\clearpage
\section{Additional Results}

\paragraph{Predicted probabilities.}
We plot a histogram of predicted probability distribution in grey with the left $y$-axis and a scatter plot of the proportion of flip counts for each bin aligned with the right $y$-axis. By overlapping the plots, we gain a comprehensive view of the model's confidence in its predictions (via the histogram) and the areas where the model predictions are most prone to change (scatter plot of flips). Notice that the scale is different between the histogram and the flip counts. The top row corresponds to the DNN experiments and the bottom row are the UA-DNN experiments. Each column represents an experiment. From the left, we show results for predictive multiplicity, large dataset update, and small dataset update.

\begin{table}[t!]
    \centering
    \resizebox{0.8\linewidth}{!}{\begin{tabular}{ll | c | c }
    \toprule
    \cell{l}{Dataset } &
    \cell{l}{Model } &
    \cell{l}{Predictive Churn \\ (Large Data Update)}  &
    \cell{l}{Predictive Churn  \\ (Small Data Update)} 
    \\ 
    \toprule
         \textds{Adult} & DNN & 
         0.58  
         & 0.73   \\
         \textds{Credit} & DNN 
         & 0.47   % sampled
         &  0.85   \\ % drop ex
         \textds{HDMA} & DNN 
         & 0.68   % sampled
         &  0.78  \\ % drop ex
         \textds{mammo} & DNN 
         & 0.20   % sampled
         &  0.50  \\ % drop ex
         \midrule
         \textds{Adult} & UA-DNN 
         & 0.64   % sampled
         & 0.91   \\ % drop ex
         \textds{Credit} & UA-DNN 
         & 0.67   % sampled
         & 0.81   \\
         \textds{HDMA} & UA-DNN 
         & 0.44   % sampled
         & 0.81   \\ 
         \textds{mammo} & UA-DNN 
         & 0.73   % sampled
         & 1.0   \\ \\
    \end{tabular}}
    
    \caption{This table shows the $\epsilon$-Rashomon unstable set tends to contain many of the examples within the churn unstable set. We report common flipped examples across different
    experiments i.e. the proportion of churned examples that are included in the $\epsilon$-Rashomon unstable set.}\label{tab:flip_idx}
\end{table}

\begin{figure}[ht!]
     \centering
    %  \scriptsize
     \resizebox{\linewidth}{!}{%
     \begin{tabular}{ccc}
     \addPredProbsPlot{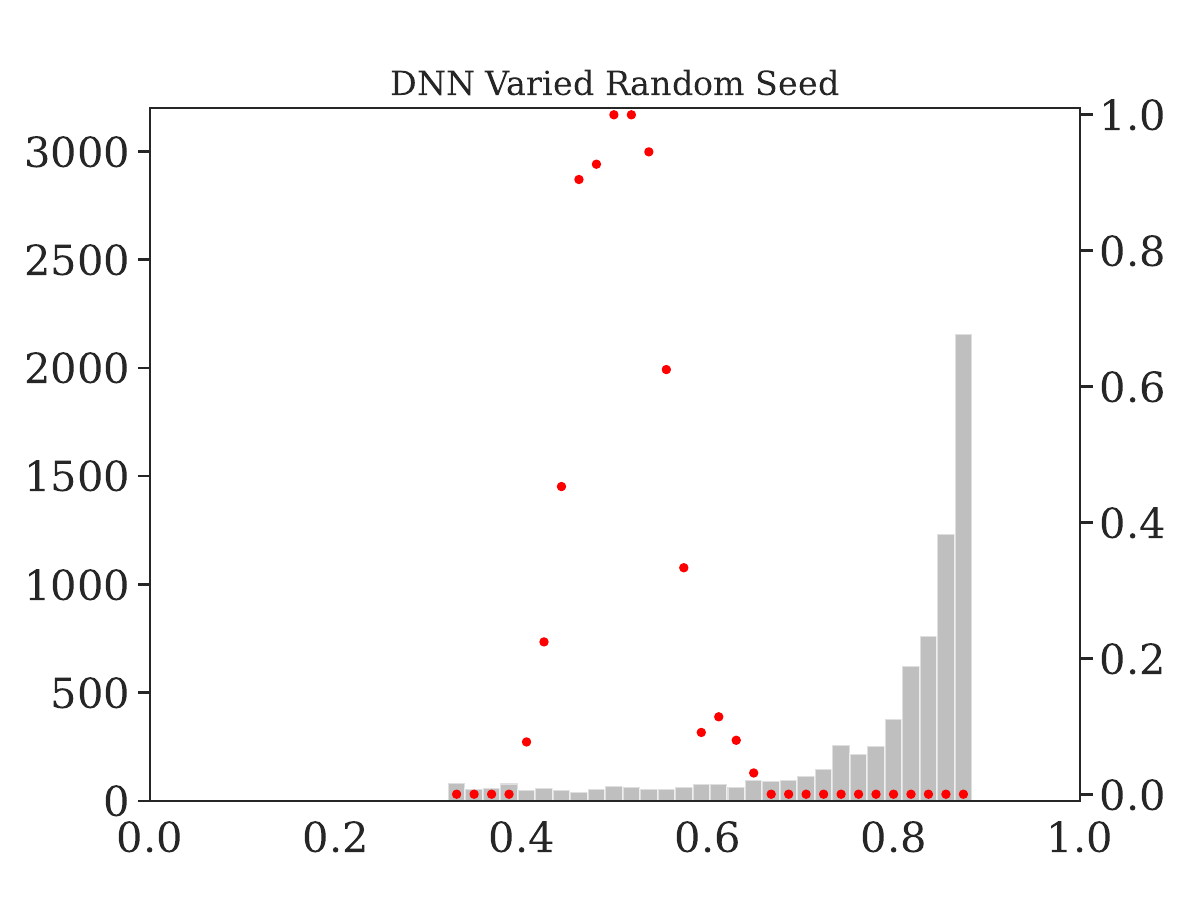} &
     \addPredProbsPlot{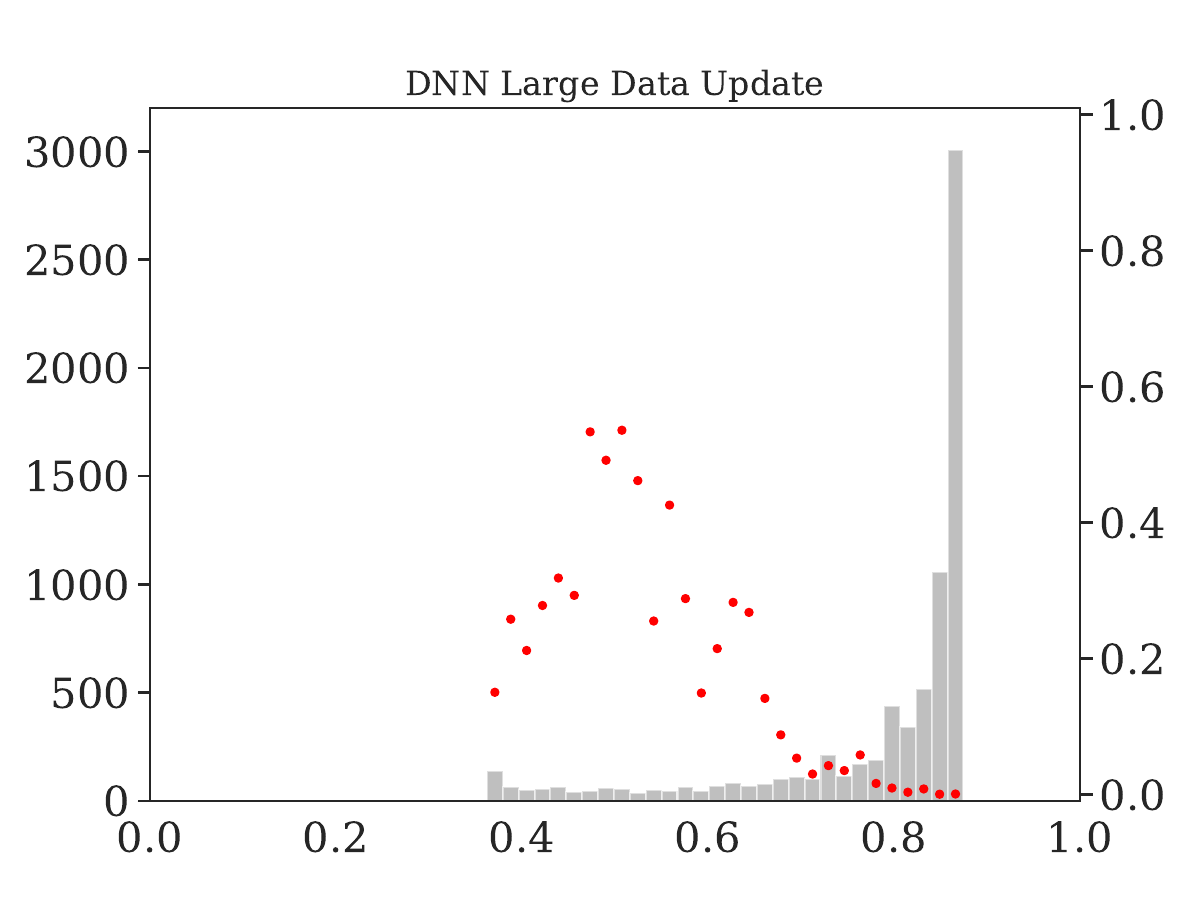} &
     \addPredProbsPlot{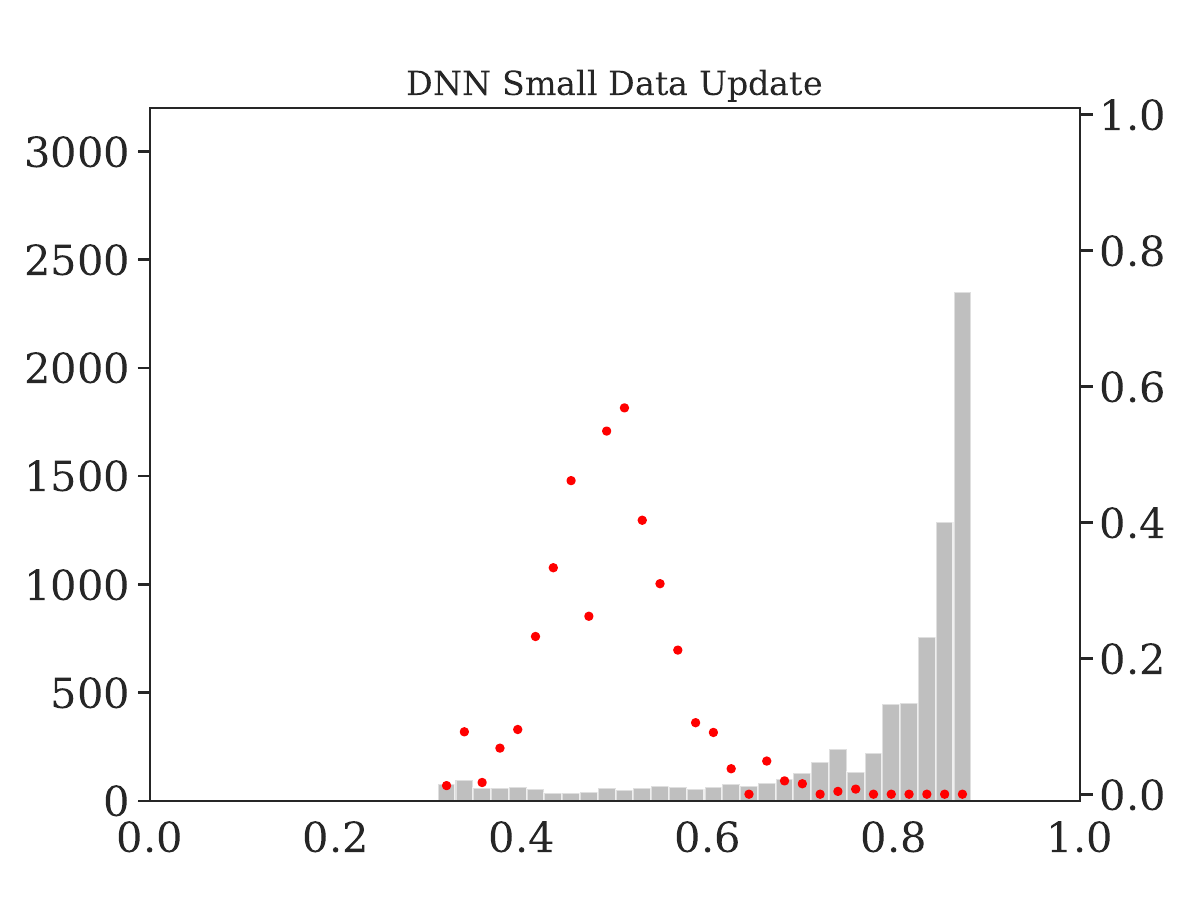}  \\
    \addPredProbsPlot{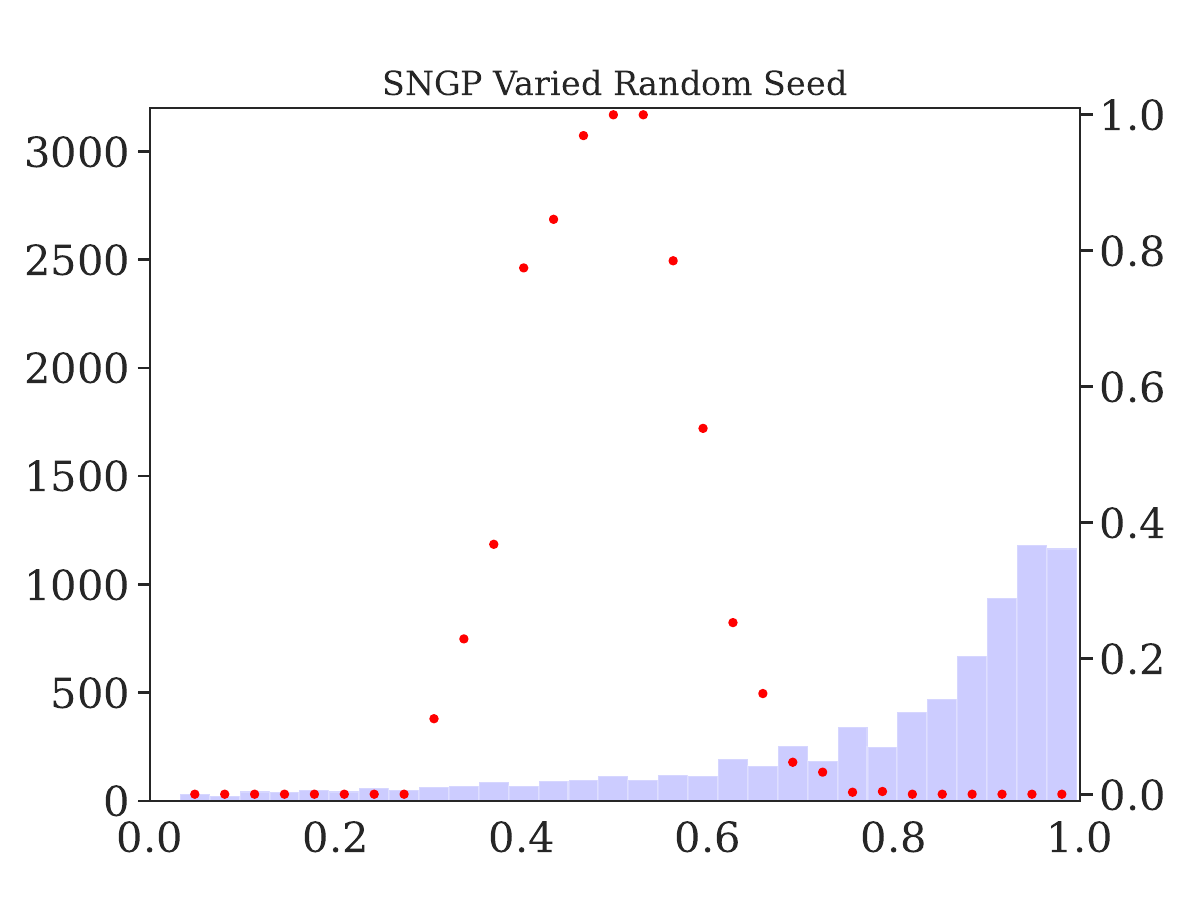} &
     \addPredProbsPlot{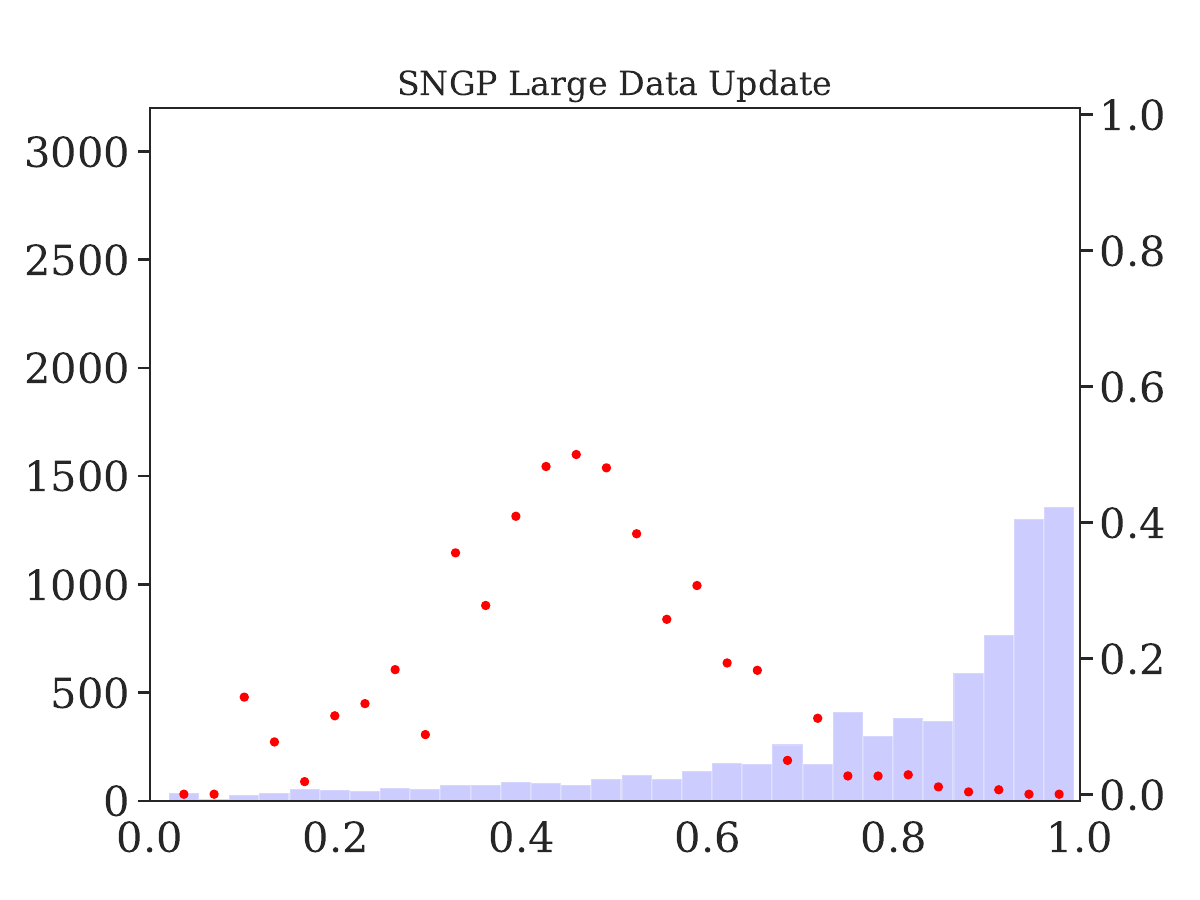} &
     \addPredProbsPlot{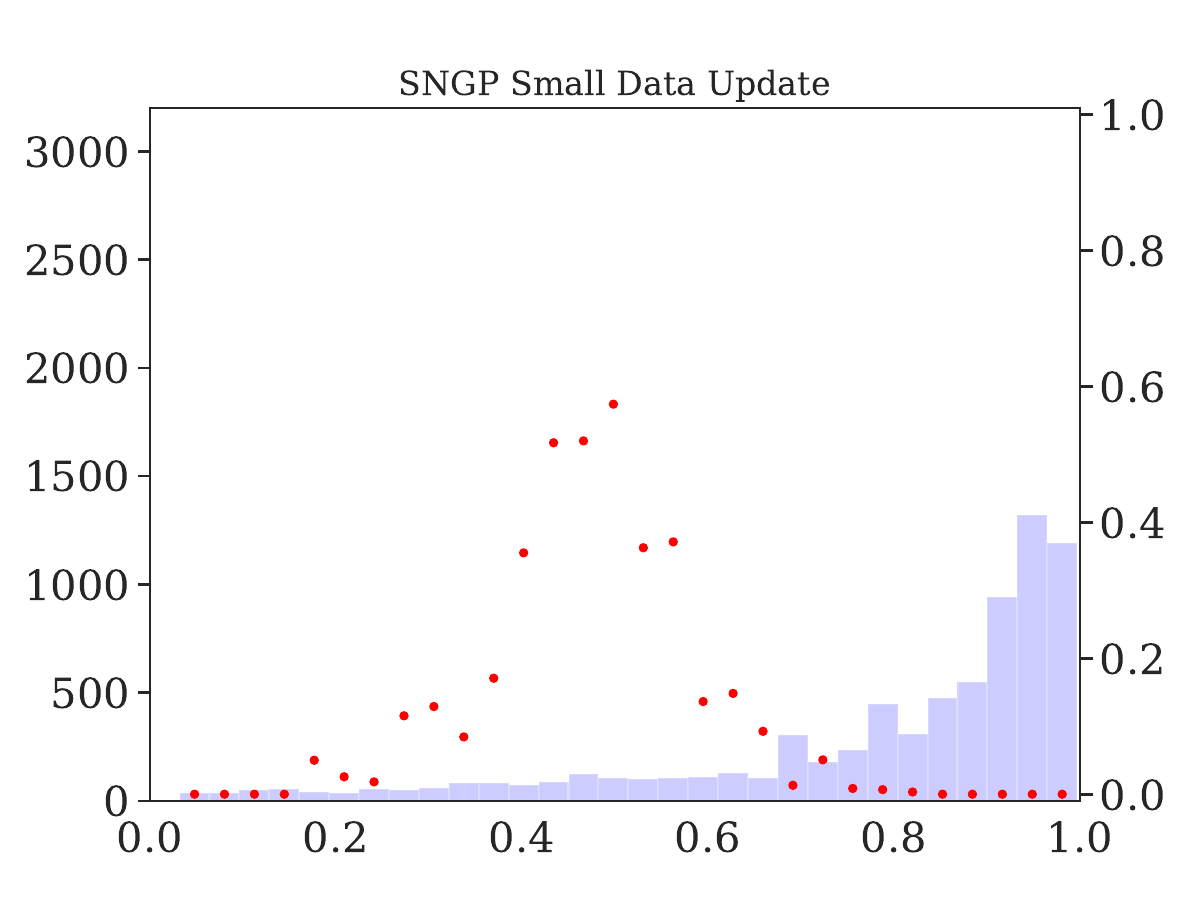} \\
     \end{tabular}%
     }
      \caption{Predicted probability distributions for Credit Dataset.} \label{fig::pred_probs_credit}
\end{figure}

\begin{figure}[ht!]
     \centering
    %  \scriptsize
     \resizebox{\linewidth}{!}{%
     \begin{tabular}{ccc}
     \addPredProbsPlot{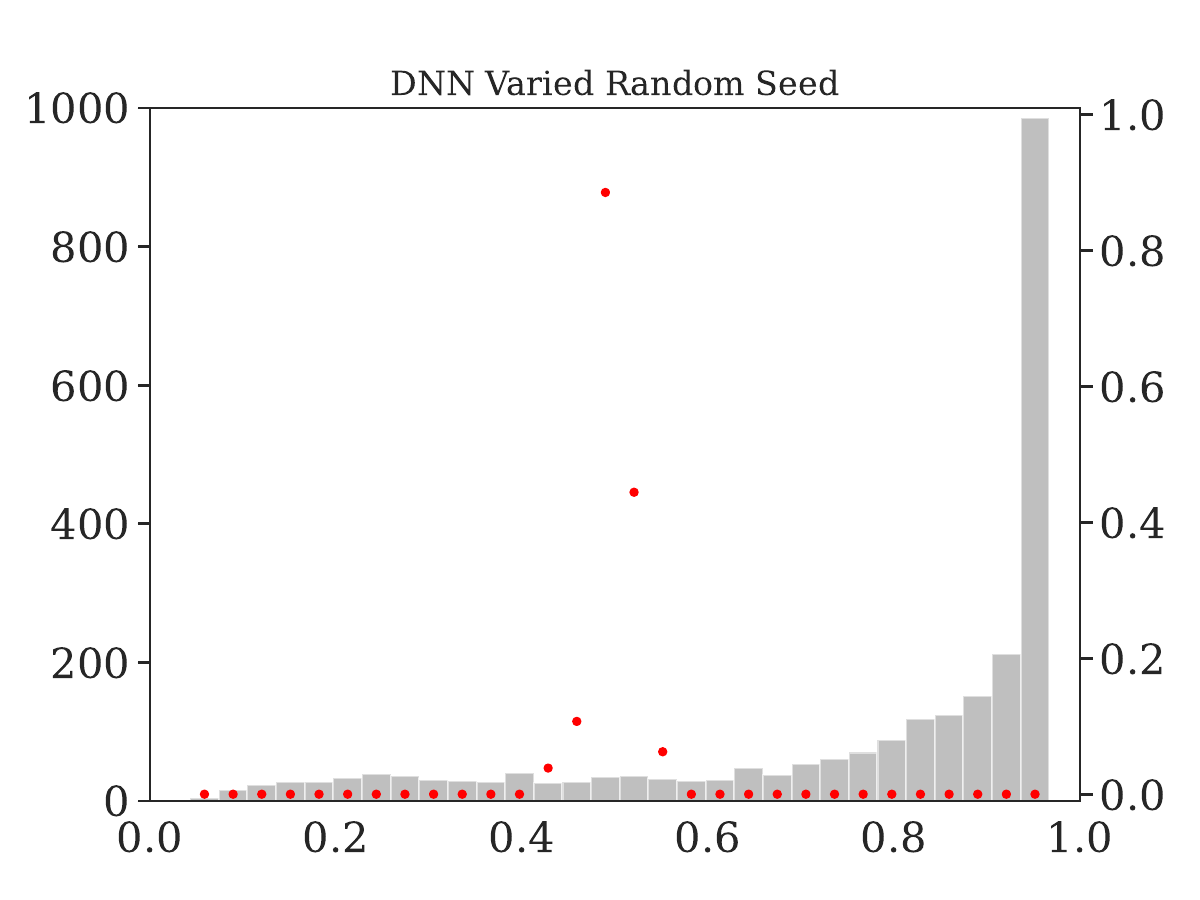} &
     \addPredProbsPlot{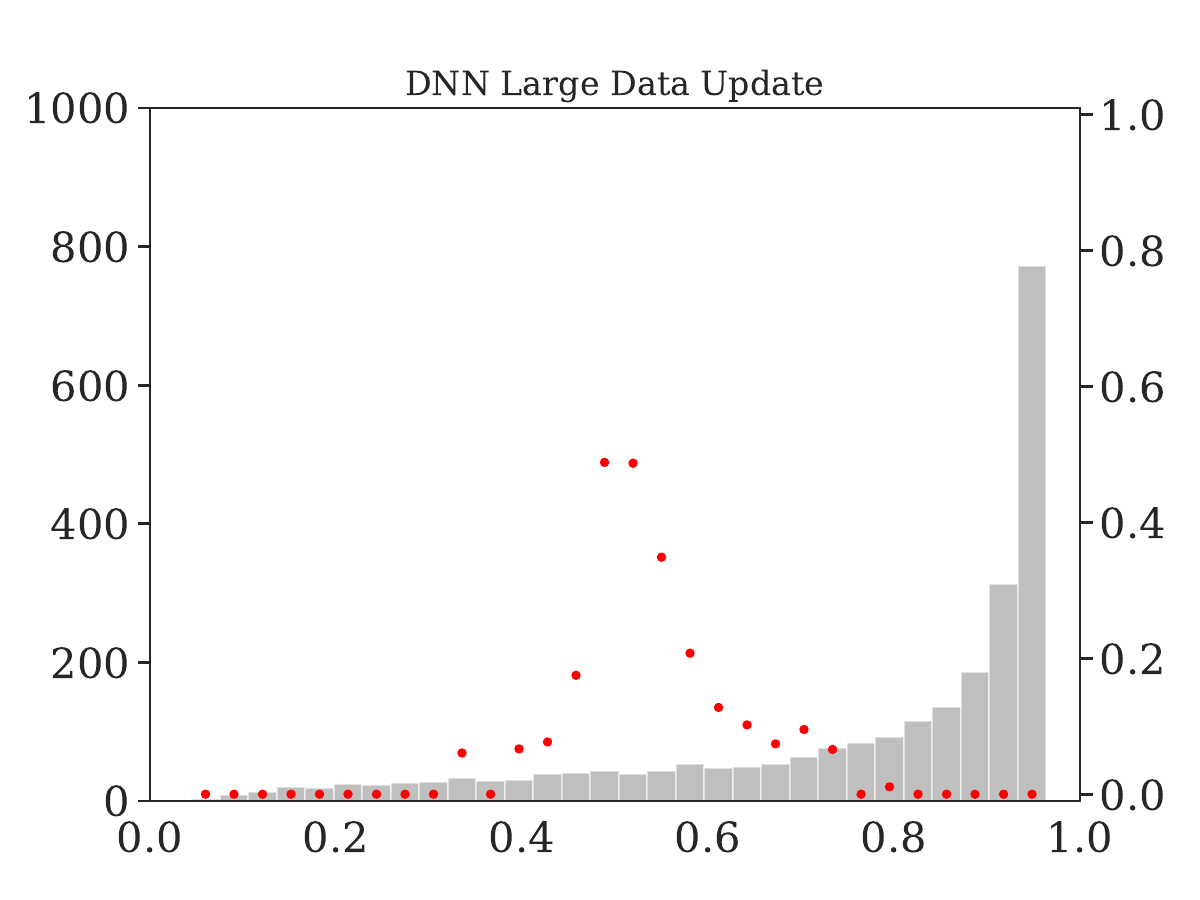} &
     \addPredProbsPlot{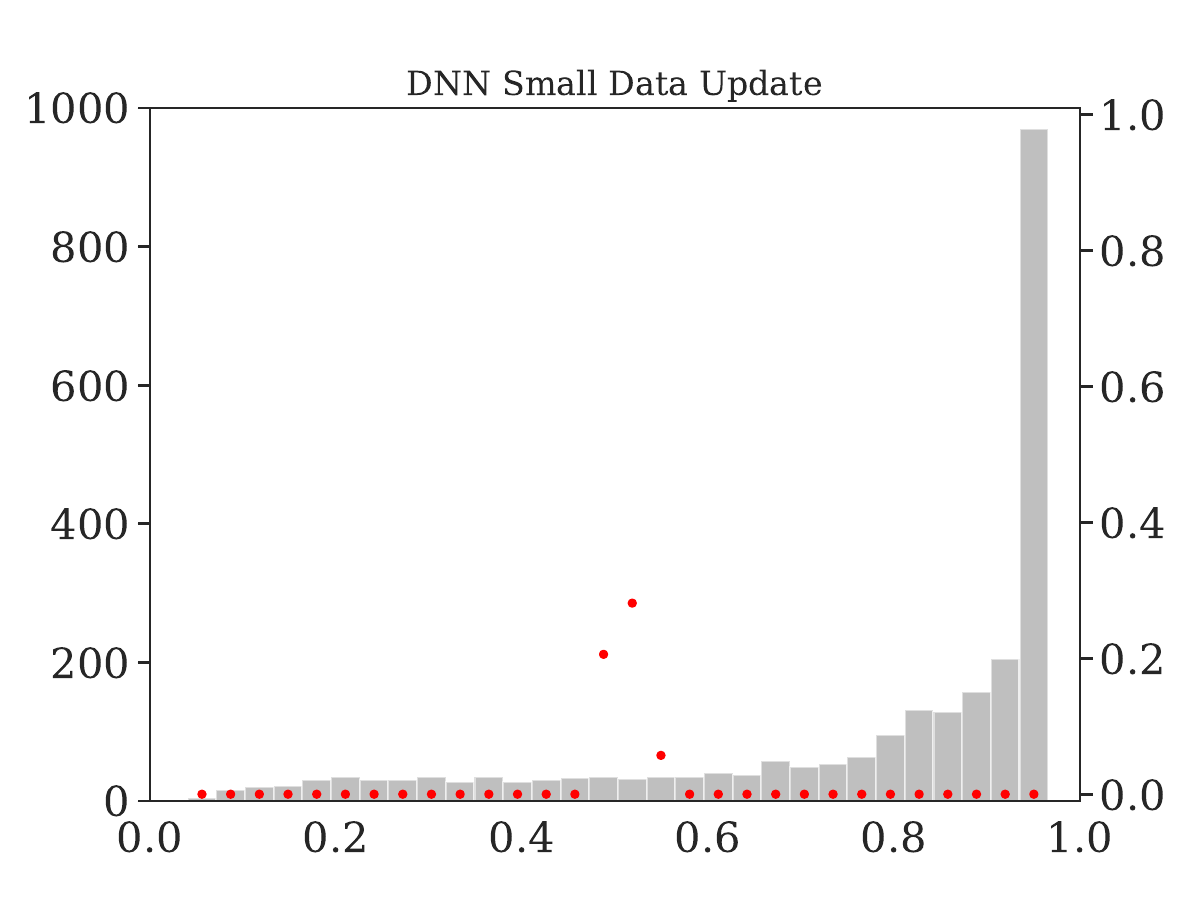}  \\
    \addPredProbsPlot{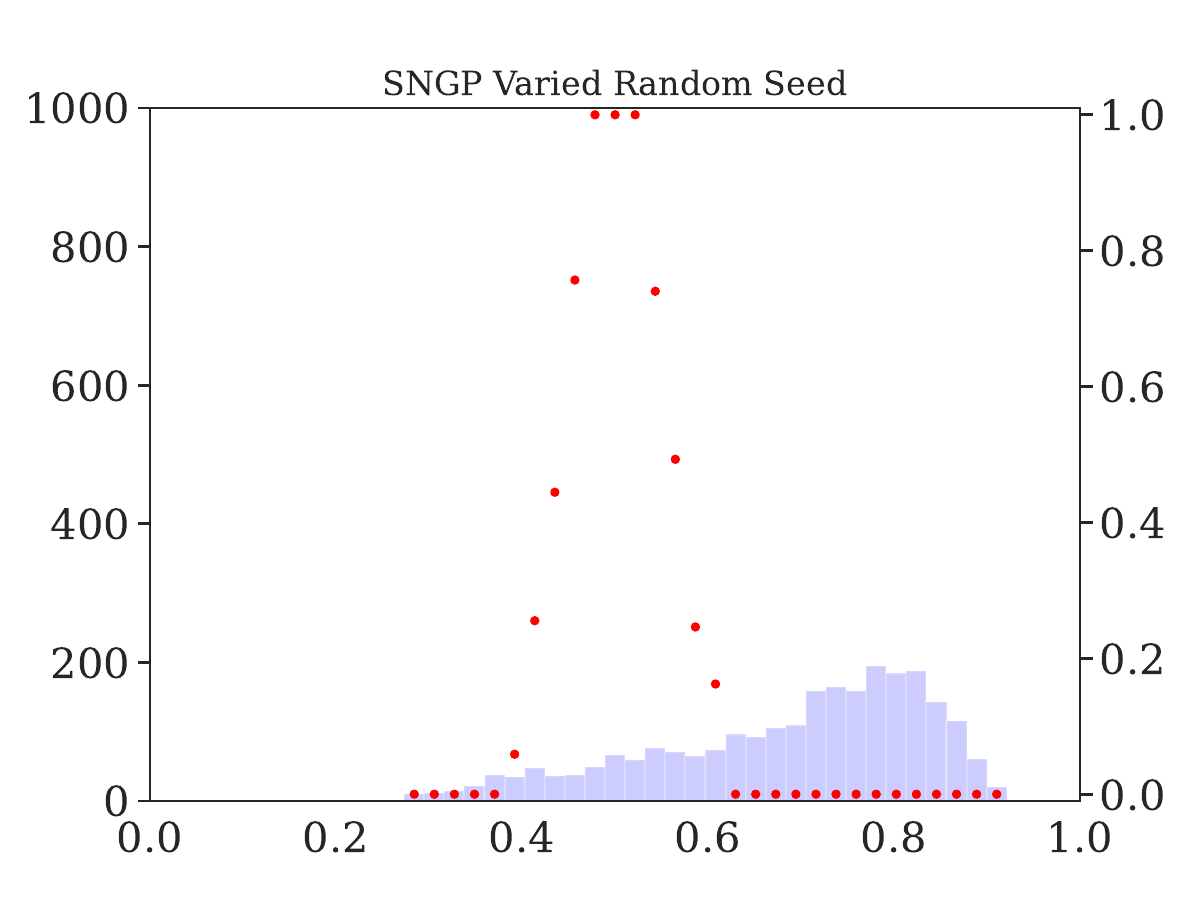} &
     \addPredProbsPlot{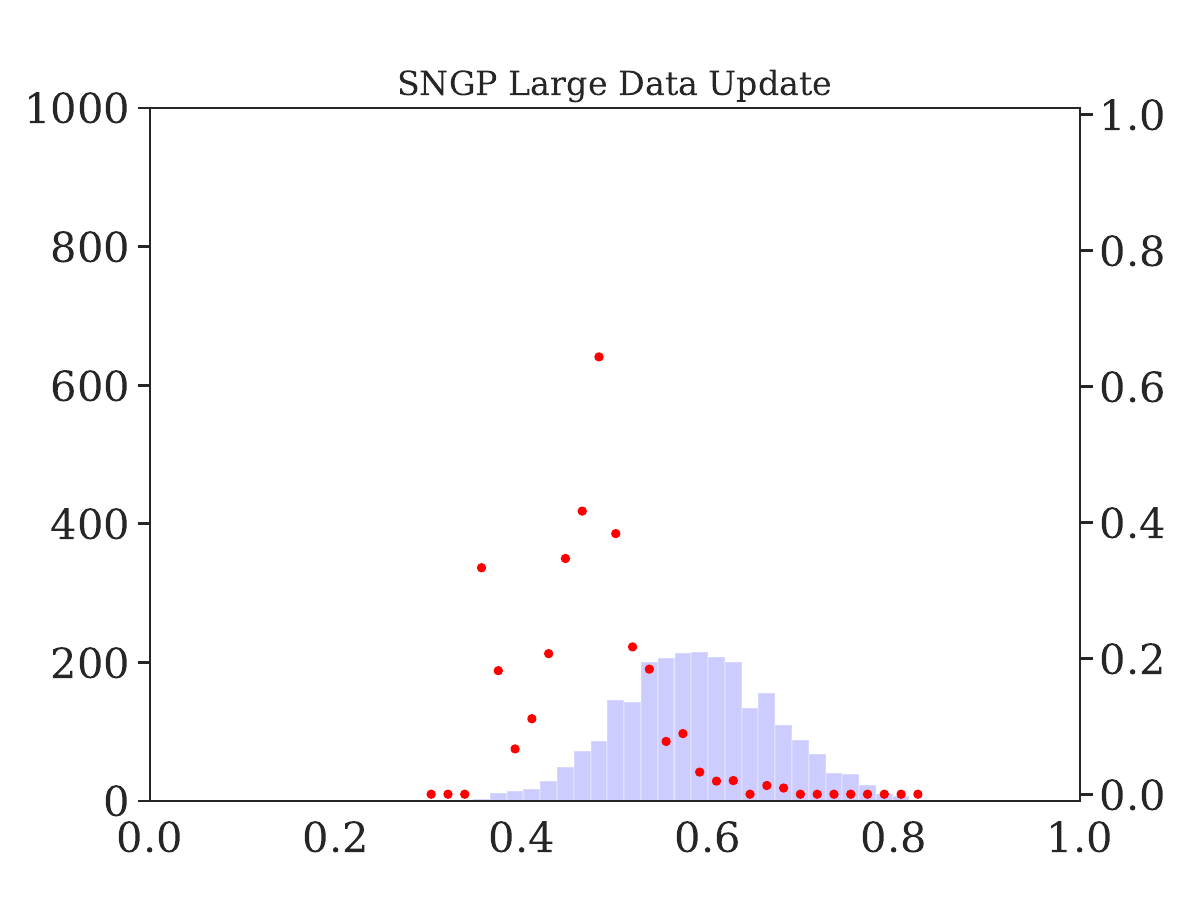} &
     \addPredProbsPlot{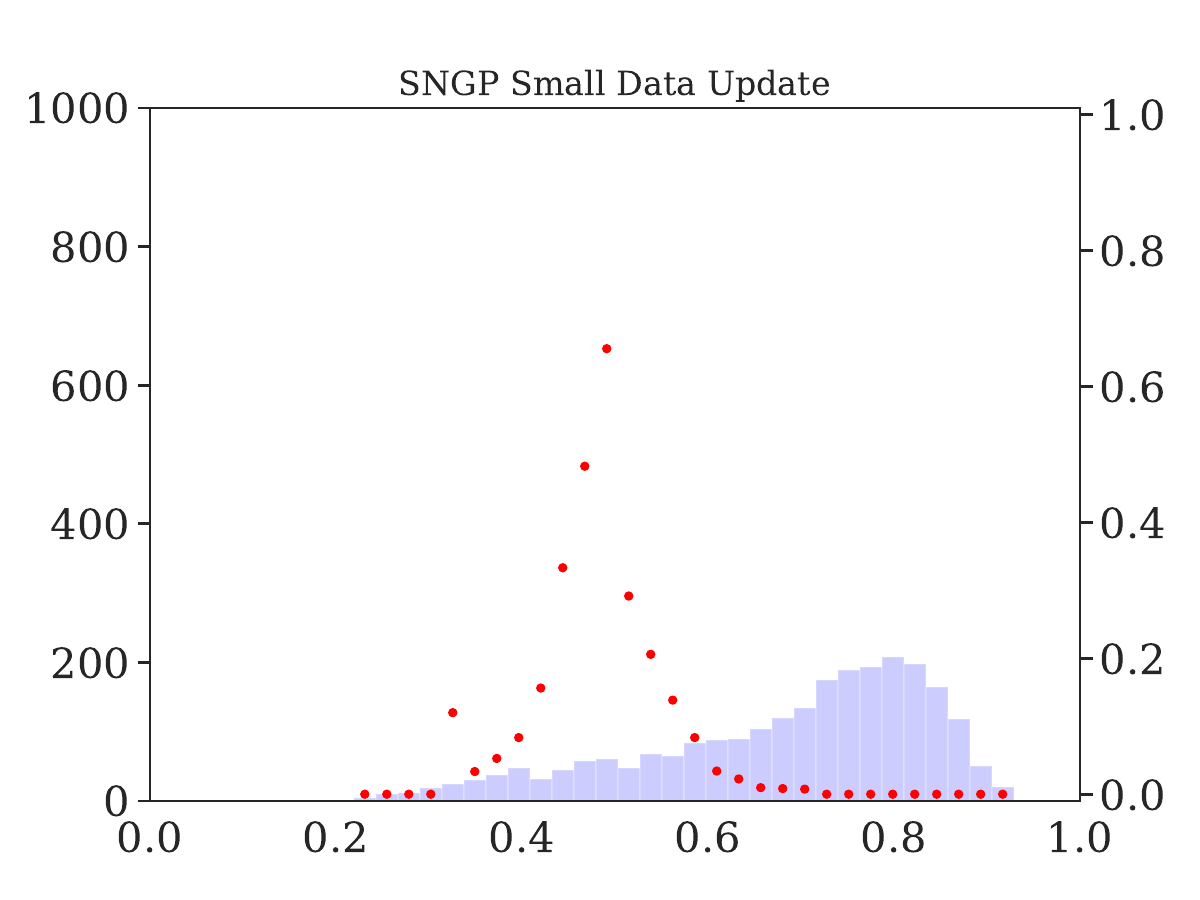} \\
     \end{tabular}%
     }
      \caption{Predicted probability distributions for HDMA Dataset.} \label{fig::pred_probs_hdma}
\end{figure}

\begin{figure}[ht!]
     \centering
    %  \scriptsize
     \resizebox{\linewidth}{!}{%
     \begin{tabular}{cc}
     \includegraphics[]{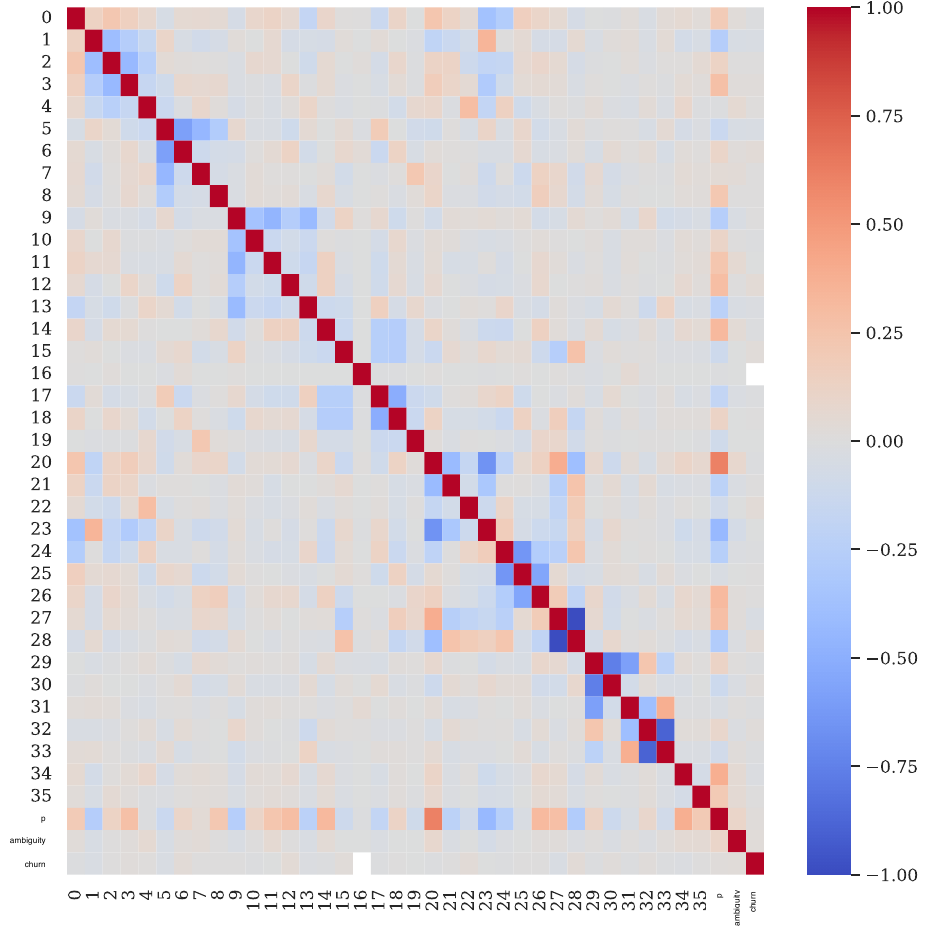} &
     \includegraphics[]{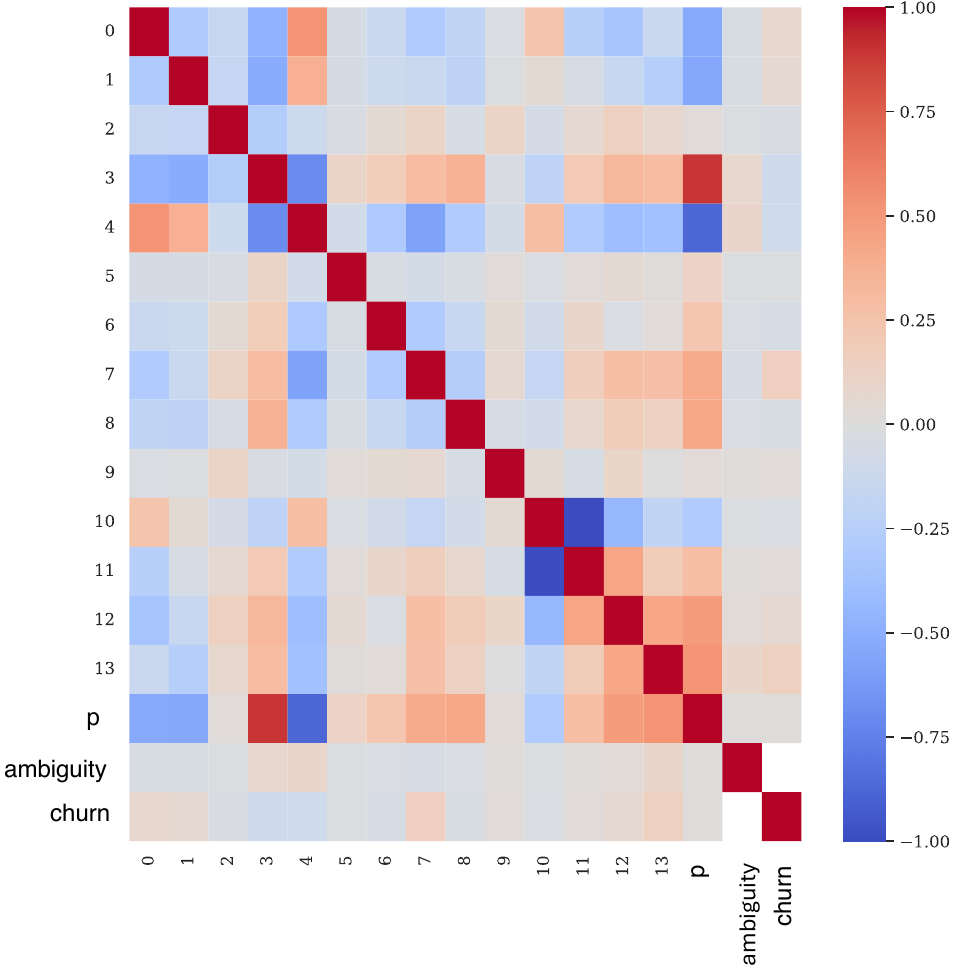}  \\
    \includegraphics[]{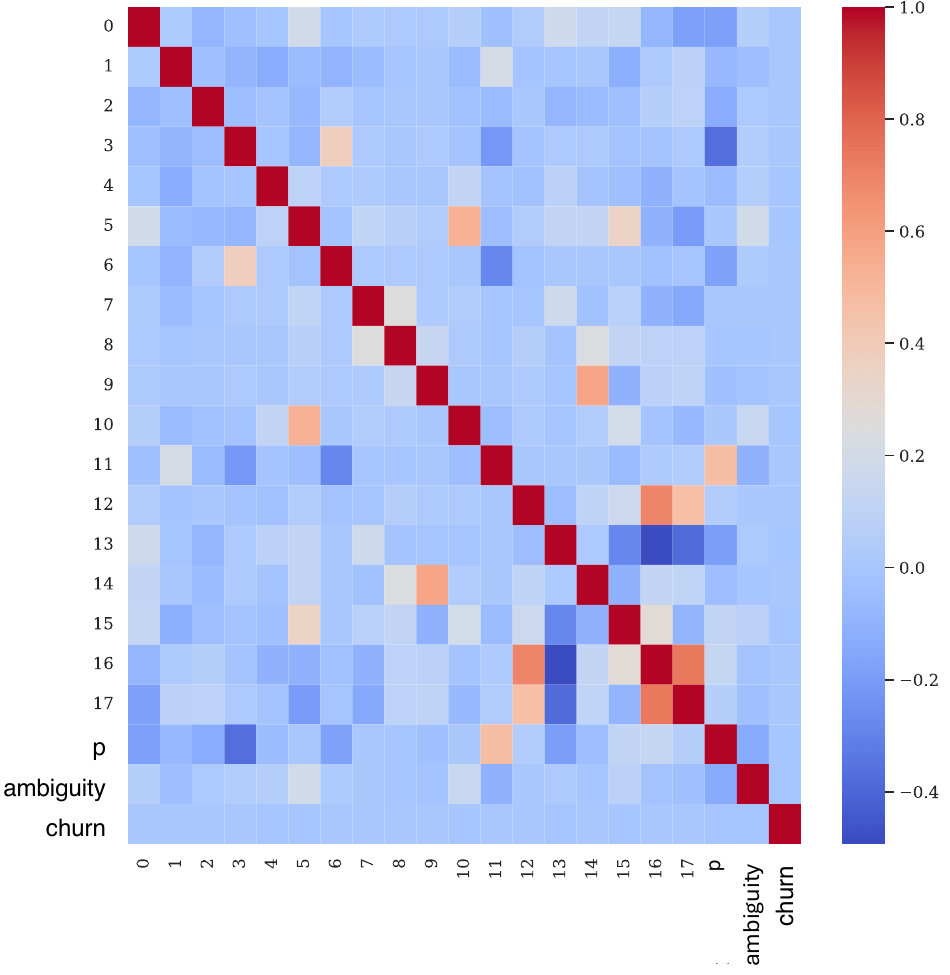} &
     \includegraphics[]{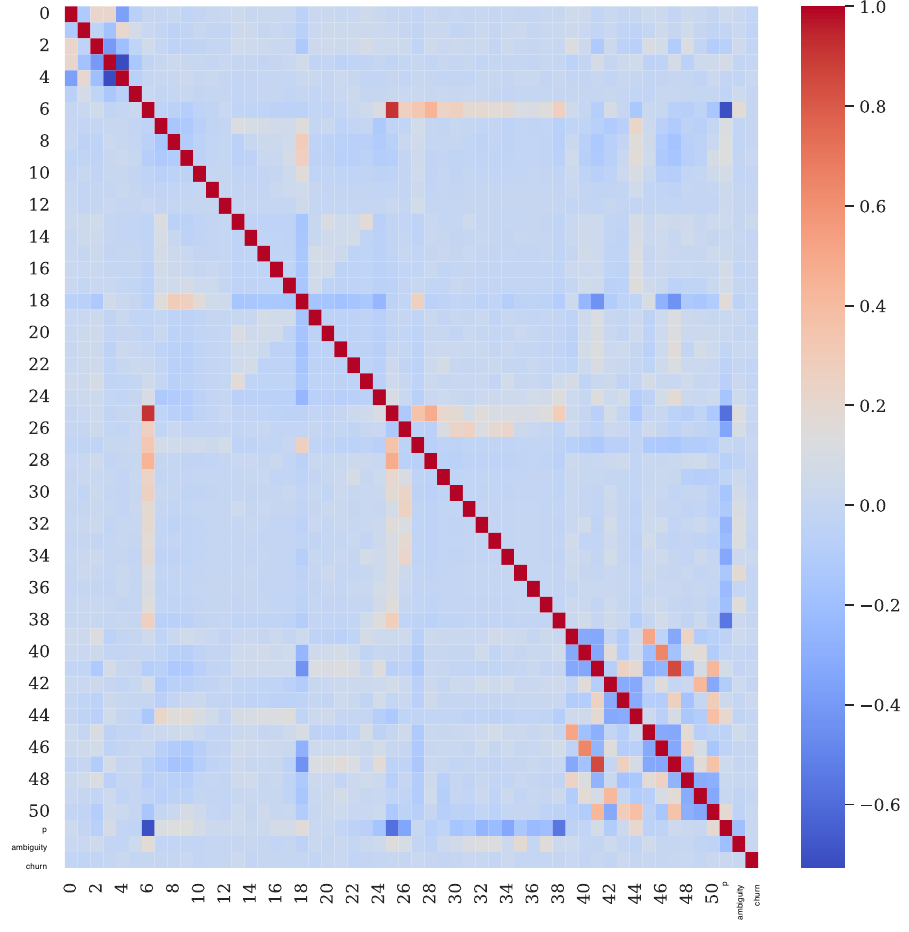}  \\
     \end{tabular}}
      \caption{Pearson correlation between features, predicted probabilities ($p$), ambiguity indiciator and churn indicator. Top left is \textit{adult}, top right is \textit{mammo}, bottom left is \textit{hmda}, bottom right is \textit{credit}. Results shown for DNN model.} \label{fig::corr_DNN}
\end{figure}

\end{document}